\documentclass[a4paper]{article}
\usepackage{amssymb}
\usepackage[]{graphicx}
\usepackage{subcaption}
\usepackage{amsmath}
\usepackage{xspace}
\usepackage{bm}
\usepackage{amsthm}
\usepackage{comment}
\usepackage{afterpage}
\usepackage{indentfirst}
\usepackage{enumerate}
\usepackage{algorithm,algorithmic}
\usepackage[numbers,sort&compress]{natbib}
\usepackage{color}
\usepackage{mathtools}
\newtheorem{theorem}{Theorem}
\newtheorem{lemma}[theorem]{Lemma}
\newtheorem{proposition}[theorem]{Proposition}

\theoremstyle{definition}
\newtheorem{definition}{Definition}

\theoremstyle{definition}

\newtheorem{remark}{Remark}

\usepackage{hyperref}

\newcommand{\nbb}{\mathbb{N}}
\newcommand{\bz}{\mathbf{z}}

\newcommand{\bw}{\mathbf{w}}
\newcommand{\bW}{\mathbf{W}}
\newcommand{\bV}{\mathbf{V}}

\newcommand{\var}{\mathrm{Var}}
\newcommand{\fcal}{\mathcal{F}}
\newcommand{\kcal}{\mathcal{K}}
\newcommand{\ibb}{\mathbb{I}}

\newcommand{\xcal}{\mathcal{X}}
\newcommand{\wcal}{\mathcal{W}}
\newcommand{\vcal}{\mathcal{V}}
\newcommand{\rcal}{\mathcal{R}}

\newcommand{\hcal}{\mathcal{H}}
\newcommand{\bx}{\mathbf{x}}
\newcommand{\bX}{\mathbf{X}}
\newcommand{\bu}{\mathbf{u}}

\newcommand{\zcal}{\mathcal{Z}}

\newcommand{\ycal}{\mathcal{Y}}

\newcommand{\gcal}{\mathcal{G}}

\newcommand{\ebb}{\mathbb{E}}
\newcommand{\pbb}{\mathbb{P}}

\newcommand{\bv}{\mathbf{v}}

\newcommand{\rbb}{\mathbb{R}}

\numberwithin{equation}{section}

\title{Towards Initialization-dependent and Non-vacuous  Generalization Bounds for Overparameterized Shallow Neural Networks}

\author{Yunwen Lei, Yufeng Xie$^{1}$\\[1.2pt]
  $^1$Department of Mathematics, The University of Hong Kong\\[1.2pt]
  \texttt{leiyw@hku.hk, yufeng.xie@connect.hku.hk}
}

\begin{document}

\maketitle
\begin{abstract}
Overparameterized neural networks often show a benign overfitting property in the sense of achieving excellent generalization behavior despite the number of parameters exceeding the number of training examples. A promising direction to explain benign overfitting is to relate generalization to the norm of distance from initialization, motivated by the empirical observations that this distance is often significantly smaller than the norm itself. However, the existing initialization-dependent complexity analyses measure the distance from initialization by the Frobenius norm, and often imply vacuous bounds in practice for overparamterized models. In this paper, we develop initialization-dependent complexity bounds for shallow neural networks with general Lipschitz activation functions. Our bounds depend on the path-norm of the distance from initialization, which are derived by introducing a new peeling technique to handle the challenge along with the initialization-dependent constraint. We also develop a  lower bound tight up to a constant factor. Finally, we conduct empirical comparisons and show that our generalization analysis implies non-vacuous bounds for overparameterized networks.
\end{abstract}

\medskip


\medskip

\section{Introduction}

{D}{eep} neural networks (DNNs) have found tremendous successful applications in various fields of science and engineering, leading to rapid progress of artificial intelligence~\citep{lecun2015deep}. A notable property of DNNs is that the models are often overparameterized~\citep{zhang2021understanding,allen2019learning}, meaning that the number of parameters far exceeds the number of training examples. By the conventional wisdom of statistical learning theory~\citep{vapnik2013nature}, these overparameterized models are doomed to overfit the training examples, yielding vacuous generalization bounds. However, various studies show that overparameterized DNNs can be trained to overfit the training examples while still making accurate predictions on testing examples, showing an interesting benign overfitting phenomenon~\citep{bartlett2020benign,neyshabur2019towards}.

The benign overfitting phenomenon motivates a lot of generalization studies of DNNs~\citep{bartlett2002rademacher,bartlett2017spectrally,neyshabur2015norm,golowich2018size}. A representative result in this direction is to build norm-based generalization bounds, meaning that the generalization bounds depend on the norm of weight matrices instead of the number of parameters. For example, the work~\citep{golowich2018size,neyshabur2015norm} used covering numbers and Rademacher complexities to study generalization in terms of norm of weights. Their analyses critically depend on the homogeneity of ReLU activation function, and therefore the results apply only to ReLU networks.

Furthermore, it is empirically observed that the neural networks trained by gradient methods often stay close to the initialization point, and the distance from the initialization can be significantly smaller than the norm of the weights~\citep{dziugaite2017computing,nagarajan2019generalization,neyshabur2019towards}.
Then, generalization bounds in terms of distance from initialization are significantly sharper than those in terms of weight norms. This observation has motivated the recent growing interests in developing initialization-dependent generalization bounds~\citep{bartlett2017spectrally,magen2023initialization,neyshabur2019towards,daniely2024sample,neyshabur2019towards}.

The work~\citep{bartlett2017spectrally} used covering numbers to develop generalization bounds depending on $\|\bW-\bW^{(0)}\|_{1,2}$, where $\bW$ and $\bW^{(0)}$ are the weight matrices for the considered DNNs and the initialization point, respectively, and $\|\cdot\|_{1,2}$ denotes the $(1,2)$-norm of matrices. However, due to the $(1,2)$-norm, the resulting bounds still enjoy a heavy dependency on the width. Several works try to improve this dependency by introducing other techniques. For example, the work~\citep{neyshabur2019towards} introduced a decomposition approach to study ReLU shallow neural networks (SNNs), the work~\citep{magen2023initialization} related the complexity of SNNs with smooth activation functions to that of Lipschitz function classes, and the work~\citep{daniely2024sample} studied the sample complexity of SNNs from the perspective of approximate description length. While these analyses imply bounds depending on the distance from initialization, they all use Frobenius norm to measure this distance. For neural networks, a more effective norm to measure the complexity of the function spaces is the path-norm~\citep{neyshabur2015path,parhi2021banach}. However, to the best of our knowledge, there are no generalization analyses of SNNs in terms of the path-norm which take into account the effect of the initialization.
Furthermore, the analysis in \citep{magen2023initialization} imposes a smoothness assumption on the activation function, while the analysis in \citep{neyshabur2019towards} considers only the ReLU activation. Moreover, the bounds in \citep{magen2023initialization,daniely2024sample} are data-independent since they involve the largest norm of inputs. Due to these issues, the existing initialization-dependent analyses often imply vacuous generalization bounds for overparameterized models in practice. Therefore, there is still a gap between theoretical analysis and empirical observations.

The above discussions motivate us to study the following natural questions.

\medskip

{\itshape
  Can we develop initialization-dependent generalization bounds that remain valid for overparameterized networks? Can our generalization bounds be expressed in terms of the path-norm rather than the Frobenius norm? Can the resulting analysis apply to general Lipschitz activation functions and imply data-dependent bounds? Can these bounds be non-vacuous for overparamterized models in empirical analyses?
  }

\medskip

In this paper, we aim to answer the above questions affirmatively  by presenting sharper initialization-dependent generalization analysis of SNNs in terms of path-norm. We summarize our contributions as follows.
\begin{enumerate}
  \item We present the first initialization-dependent Rademacher complexity bounds for SNNs with general Lipschitz activation functions. Our bounds are expressed by the path-norm and are data-dependent with a logarithmic dependency on the width, while the existing initialization-dependent bounds are expressed by the Frobenius norm~\citep{neyshabur2019towards,magen2023initialization,neyshabur2018pac,daniely2024sample}, data-independent~\citep{magen2023initialization,daniely2024sample}, enjoy a square-root dependency on the width~\citep{bartlett2017spectrally,neyshabur2019towards,neyshabur2018pac} or require smoothness assumptions on the activation function~\citep{magen2023initialization}. Based on this, we present generalization bounds which are significantly sharper than existing results, especially if the width is large.
  \item We also develop lower bounds for Rademacher complexity of SNNs, which match our upper bounds up to a constant factor. This shows the tightness of our complexity analysis. Unlike the existing lower bounds focusing on $\bW^{(0)}=0$~\citep{neyshabur2019towards,bartlett2017spectrally}, our lower bounds are also established for initialization-dependent SNNs.
  \item We introduce new techniques in the Rademacher complexity analysis to handle the initialization-dependent constraint. Our key idea is to introduce several auxiliary function spaces by fully exploiting the initialization-dependent constraint. We then relate the Rademacher complexity of SNNs to a finite union of these function spaces, which can be estimated by existing structural results on Rademacher complexities~\citep{maurer2014inequality}.
  \item We conduct an empirical analysis to illustrate the behavior of our generalization bounds\footnote{{Code for empirical studies is available at https://github.com/xxyufeng/Generalization-SNNs.}}. The experimental comparison shows that our bound achieves a consistent and significant improvement over existing bounds. In particular, empirical analysis shows that our generalization bounds are non-vacuous in the sense of being smaller than $1$ even the networks are highly overparameterized.
\end{enumerate}

We organize the paper as follows. We review the related work in Section~\ref{sec:work}, and introduce the problem setup in Section~\ref{sec:problem}. We present our main results on Rademacher complexities and generalization bounds in Section~\ref{sec:result}. We collect the proofs in Section~\ref{sec:proof-gen}, and conclude the paper in Section~\ref{sec:conclusion}.

\section{Related Work\label{sec:work}}
In this section, we present the related work on generalization analyses of neural networks. We divide our discussions into three categories: complexity approach, stability approach and other approach.

\medskip

\noindent \textbf{Complexity approach}. A popular approach to study the generalization of neural networks is to study the complexity of the corresponding hypothesis spaces. Classical studies analyzed the complexity by counting the number of parameters~\citep{bartlett2019nearly}. More advanced techniques considered the norm-based complexity analyses, using the techniques of covering numbers~\citep{bartlett2017spectrally} and Rademacher complexities~\citep{bartlett2002rademacher,golowich2018size,neyshabur2015norm,yang2024nonparametric}. Most of these work considered initialization-independent hypothesis spaces, and developed generalization bounds depending on the product of the Frobenius norm of the weight matrices~\citep{golowich2018size,neyshabur2015norm,ledent2021norm}. Motivated by the empirical observation that $\|\bW^{(t)}-\bW^{(0)}\|_F\ll \|\bW^{(t)}\|_F$ for gradient descent iterates $\{\bW^{(t)}\}_t$, there has been growing interest to develop complexity bounds in terms of the distance from initialization. The work~\citep{bartlett2017spectrally} used covering numbers to derive complexity bounds that depend on the $(1,2)$-norm of the distance from initialization. The work~\citep{neyshabur2019towards} considered the ReLU SNNs and developed generalization bounds depending on the distance from initialization and also the spectral norm of the initialization matrix. Similar complexity bounds were also derived for SNNs with Lipschitz and smooth activations~\citep{magen2023initialization}. The recent work~\citep{lei2026optimization,li2025optimal} developed Rademacher complexity bounds of ReLU networks in a lazy training setting based on the neural tangent kernels.  These complexity bounds were widely used to study the generalization behavior of optimization methods applied to neural networks~\citep{du2018gradient,arora2019fine,ji2019polylogarithmic,chen2021much,allen2019learning,ding2025semi,weinan2022barron,chen2020generalized}.

\medskip

\noindent\textbf{Stability approach}. Another effective approach to study generalization is to consider the stability of an algorithm with respect to the perturbation of the training dataset~\citep{bousquet2002stability,hardt2016train,lei2020fine}. The seminal work showed that the weak convexity of neural networks scales as $\frac{1}{\sqrt{m}}$~\citep{liu2020linearity}, a result later used to study the generalization of gradient methods to train SNNs under the square loss function~\citep{richards2021stability,wang2025generalization}. The work~\citep{taheri2024generalization} showed that neural networks attain a stronger self-bound weak convexity, which plays an important role in the generalization analysis of neural networks with polylogarithmic width~\citep{taheri2024generalization,taheri2025sharper}. The algorithmic stability of transformers and GANs has also been studied recently~\citep{ji2021understanding,li2023transformers,deora2024optimization}.

\medskip
\noindent\textbf{Other approach}. There are also other approaches to develop generalization bounds for DNNs, including the PAC-Bayesian analysis~\citep{neyshabur2018pac}, compression analysis~\citep{arora2018stronger,suzuki2020compression}, optimal transport analysis~\citep{chuang2021measuring}, approximate description length~\citep{daniely2024sample} and information theoretical analysis~\citep{he2025information}.
\section{Problem Setup\label{sec:problem}}
Let $\pbb$ be a probability measure defined on a sample space $\zcal=\xcal\times\ycal$, where $\xcal\subseteq\rbb^d$ is an input space and $\ycal$ is an output space. We consider a multi-class classification task where $\ycal=\{1,\ldots,c\}$ with $c$ being the number of class labels. Let $S=\{\bz_1,\ldots,\bz_i\}$ with $\bz_i=(\bx_i,y_i)$ be a dataset drawn independently from $\pbb$, based on which we want to build a model $g:\xcal\mapsto\rbb^c$ for prediction. The model $g$ is a vector-valued function, and the $k$-th component $g_k$ can be interpreted as the score function for the $k$-th label, based on which we can do the prediction by  $\bx\gets\arg\max_{k}g_k(\bx)$.
The performance of a model $g$ on an example $(\bx,y)$ is measured by a loss function $\ell_y:\rbb^c\mapsto\rbb_+$, i.e., $\ell_y(g(\bx))$. Then, we can define the empirical and population risk as follows
\[
F_S(h)=\frac{1}{n}\sum_{i=1}^{n}\ell_{y_i}(h(\bx_i))\quad\text{and}\quad F(h)=\ebb_{\bz}[\ell_y(h(\bx))],
\]
which measures the performance of $h$ on training and testing examples, respectively. A term of a central interest in machine learning theory is the generalization gap, which is defined as the difference between the population and empirical risks.

A popular model for prediction is the neural network.
In this paper, we focus on shallow neural networks (SNNs) with a single hidden layer.  Let $\gamma:\rbb\mapsto\rbb$ be an activation function and $m$ be the width. Then, an SNN parameterized by $\bW\in\rbb^{m\times d}$ and $\bV\in\rbb^{c\times m}$ takes the form
\begin{equation}\label{snn}
\Psi_{\bW,\bV}(\bx)=\bV\begin{pmatrix}
                                   \gamma(\bw_1^\top\bx) \\
                                   \vdots \\
                                   \gamma(\bw_m^\top\bx)
                                 \end{pmatrix},
\end{equation}
where ($\bw_j\in\rbb^d$ and $\bv_k\in\rbb^m$)
\begin{equation}\label{WVm}
  \bW=\begin{pmatrix}
        \bw_1^\top \\
        \vdots \\
        \bw_m^\top
      \end{pmatrix},\quad \bV=\begin{pmatrix}
                                \bv_1^\top \\
                                \vdots \\
                                \bv_c^\top
                              \end{pmatrix}.
\end{equation}
Assume that $|\gamma'(a)|\leq G_\gamma$ for all $a\in\rbb$, which is satisfied for popular activation functions such as the ReLU function, the sigmoid function and the  hyperbolic tangent activation function.
Let $R_W,R_V\geq0$.
We consider hypothesis spaces for SNNs of the following form
\begin{equation}\label{gcal}
\gcal:=\Big\{\bx\mapsto\Psi_{\bW,\bV}(\bx):\bW\in\wcal,\bV\in\vcal\Big\},
\end{equation}
where
  \begin{equation}\label{WV}
    \wcal=\{\bW\in\rbb^{m\times d}:\|\bW-\bW^{(0)}\|_F\leq R_W\}\quad\text{and}\quad \vcal=\{\bV\in\rbb^{c\times m}:\|\bV\|_F\leq R_V\},
  \end{equation}
and $\Psi_{\bW,\bV}$ is defined in Eq.~\eqref{snn}.
\begin{remark}[Initialization-dependency]
Note we handle the hidden-layer weights $\bW$ and top-layer weights $\bV$ differently. For the hidden-layer, we measure the complexity by the Frobenius norm of the distance from the initialization. The underlying motivation is that the distance from the initialization (for hidden layers) is often significantly smaller than the Frobenius norm for neural networks trained by optimization algorithms such gradient methods~\citep{dziugaite2017computing,nagarajan2019generalization,neyshabur2019towards}. Then, a complexity bound in terms of the distance from initialization is significantly better than that in terms of the Frobenius norm. For the top-layer, we measure the complexity by the Frobenius norm. The motivation is that the Frobenius norm and distance from initialization are quite close for the top-layer of neural networks trained by gradient methods~\citep{neyshabur2019towards}, which suggests a limited role of initialization for this layer. Indeed, the work~\citep{neyshabur2019towards,magen2023initialization,daniely2024sample} also considered similar hypothesis spaces with initialization-dependent constraints on the hidden layer, and initialization-independent constraints on the top layer.
\end{remark}

The generalization behavior of a model often depends on the complexity of the corresponding hypothesis space. A popular complexity measure is the Rademacher complexity~\citep{bartlett2002rademacher}. Since our model outputs are vector-valued, we employ vector-valued Rademacher complexities in our generalization analysis. If $c=1$, then the vector-valued Rademacher complexity reduces to the standard Rademacher complexity~\citep{bartlett2002rademacher}.
\begin{definition}[Vector-valued Rademacher complexity~\citep{maurer2016vector}]
Let $\gcal$ be a class of vector-valued functions and $c\in\nbb$ be the output dimension. Let $S=\{\bz_1,\ldots,\bz_n\}$. Then, we define the vector-valued (empirical) Rademacher complexity by
\[
\mathfrak{R}_{S}(\gcal)=\frac{1}{n}\ebb_{\bm{\sigma}}\sup_{g\in\gcal}\sum_{i=1}^{n}\sum_{k=1}^{c}\sigma_{ik}g_k(\bz_i),
\]
where $g_k$ is the $k$-th component of $g$ and $\sigma_{ik}$ are independent Rademacher variables, i.e., they take values in $\{\pm1\}$ with the same probability.
\end{definition}

We collect some notations here.
For a vector $\bw$, we denote by $\|\bw\|_2$ the Euclidean norm. For a matrix $\bW$, we denote by $\|\bW\|_F$ the Frobenius norm, by $\|\bW\|_\sigma$ the spectral norm and by $\|\bW\|_{p,q}$ the $(p,q)$ matrix norm, defined by $\|\bW\|_{p,q}=\big\|\big(\|W_{:,1}\|_p,\ldots,\|W_{:,d}\|_p\big)\big\|_q$.
For any $n\in\nbb$, denote $[n]=\{1,\ldots,n\}$. We denote $A\lesssim B$ if there exists a universal constant $C\geq0$ such that $A\leq CB$. We also denote $A\gtrsim B$ if there exists a universal constant $C\geq0$ such that $A\geq CB$. We denote $A\asymp B$ if $A\lesssim B$ and $A\gtrsim B$.
We say a function $\ell:\rbb\mapsto\rbb$ is $L$-smooth if $|\ell'(t)-\ell(t')|\leq L|t-t'|$ for all $t,t'\in\rbb$.

\section{Main Results\label{sec:result}}

\subsection{Complexity and Generalization Bounds}

In this subsection, we present our main results on Rademacher complexity and generalization bounds. Theorem~\ref{thm:rad-snn-g} gives data-dependent Rademacher complexity bounds in the sense of involving $\big(\sum_{i=1}^n\|\bx_i\|_2^2\big)^{\frac{1}{2}}$ and $\|\sum^n_{i=1}\bx_i\bx_i^\top\|_{\sigma}^{\frac{1}{2}}$. Eq.~\eqref{rad-snn-g-a} gives complexity bounds in terms of the path-norm, while Eq.~\eqref{rad-snn-g-b} gives bounds in terms of the product of the Frobenius norm. As we will show, Eq.~\eqref{rad-snn-g-a}  is essential for us to derive generalization bounds in terms of the path-norm of the distance from the initialization.
\begin{definition}[Path-norm\label{def:path}]
For any $\bW\in\rbb^{m\times d},\bV\in\rbb^{c\times m}$, we define the path-norm with a reference matrix $\bW^{(0)}$ as
\[
\kappa(\bW,\bV)= \sum_{j=1}^{m}\sum_{k=1}^c|v_{kj}|\|\bw_j-\bw_j^{(0)}\|_2.
\]
\end{definition}
\begin{remark}[Standard path-norm]
If $c=1$ and $\bW^{(0)}=0$, then we know $\kappa(\bW,\bV)=\sum_{j=1}^{m}|v_j|\|\bw_j\|_2$ (here $\bV=\bv=(v_1,\ldots,v_m)$ since $c=1$), which recovers the standard path-norm~\citep{neyshabur2015path,neyshabur2015norm,liu2024learning,yang2024nonparametric}. The standard path-norm appears as a natural complexity measure due to the positive-homogeneity of the ReLU activation function. Specifically, let $\gamma(t)=\max\{t,0\}$ and $\Psi_{\bW,\bV}(\bx)=\sum_{j=1}^{m}v_j\gamma(\bw_j^\top\bx)$. Then, we can use the property $\gamma(ta)=t\gamma(a)$ for any $t\geq0$ to derive
\begin{equation}\label{path-exp}
\Psi_{\bW,\bV}(\bx)=\sum_{j=1}^{m}v_j\|\bw_j\|_2\gamma(\bw_j^\top\bx/\|\bw_j\|_2)=\sum_{j=1}^{m}v_j\|\bw_j\|_2\gamma(\tilde{\bw}_j^\top\bx),
\end{equation}
where $\tilde{\bw}_j=\bw_j/\|\bw_j\|_2$ is of unit norm.
If we ignore $\gamma(\tilde{\bw}_j^\top\bx)$, the right-hand side is closely-related to the standard path-norm $\kappa_s(\bW,\bV):=\sum_{j=1}^{m}|v_j|\|\bw_j\|_2$.
In this way, the Rademacher complexity of SNNs can be related to the product of the path-norm and the Rademacher complexity of the space $\{\bx\mapsto\gamma(\tilde{\bw}^\top\bx):\|\tilde{\bw}\|_2\leq 1\}$, which reads as
\begin{equation}\label{rad-snn-spn}
    \rcal_S(\gcal) \leq \frac{2G_\gamma}{n}\sup_{\bW,\bV}\sum_{j=1}^{m}|v_j|\|\bw_j\|_2 (\sum^n_{i=1}\|\bx_i\|_2^2)^\frac{1}{2},
\end{equation}
where $\gcal$ is defined in Eq.~\eqref{gcal} and $G_\gamma$ denotes the Lipschitz constant of $\gamma(\cdot)$.
\begin{itemize}
  \item However, Eq.~\eqref{path-exp} only holds for the ReLU activation function, and it is challenging to use path-norm to derive meaningful complexity bounds for general Lipschitz activation functions even for initialization-independent hypothesis spaces. For example, the work~\citep{li2020complexityb} considered the hypothesis space in Eq.~\eqref{gcal} with $\bW^{(0)}=0$ and a general activation function. They considered a different path-norm $\kappa_0(\bW,\bv)=\sum_{j=1}^{m}|v_j|\big(1+\|\bw_j\|_1\big)$, and showed that the Rademacher complexity is bounded by $\zeta(\gamma)\sup_{\bW,\bv}\kappa_0(\bW,\bv)/\sqrt{n}$, where $\zeta(\gamma)$ is a measure on the regularity of the activation function~\citep{li2020complexityb}. This result cannot be applied to overparameterized models since $\sum_{j=1}^{m}|v_j|$ can be very large. Indeed, one often set $|v_j|=1/\sqrt{m}$ in the existing analyses~\citep{ji2019polylogarithmic,arora2019fine,taheri2024generalization}, and then $\sum_{j=1}^{m}|v_j|=\sqrt{m}$, which enjoys a square-root dependency on $m$.
  \item Furthermore, Eq.~\eqref{path-exp} only motivates the standard path-norm $\sum_{j=1}^{m}|v_j|\|\bw_j\|_2$ (i.e., $\bW^{(0)}=0$). It is not clear how to use the path-norm to measure the complexity of initialization-dependent spaces. Theorem~\ref{thm:rad-snn-g} addresses these challenging questions by introducing complexity bounds in terms of the path-norm of the distance from initialization. As shown in Figure~\ref{Fig_path_norm} in Section~\ref{sec:empirical}, the path-norm in Definition~\ref{def:path} can be substantially smaller than the standard path-norm (i.e., with $\bW^{(0)}=0$) in practice. This shows the generalization bounds expressed by distance from initialization can be much more effective than bounds expressed by the norm itself.
\end{itemize}
\end{remark}


\begin{theorem}[Complexity Bounds\label{thm:rad-snn-g}]
  Let $\gcal$ be defined in Eq.~\eqref{gcal} and $\gamma(\cdot)$ be $G_\gamma$-Lipschitz. Then, we have
  \begin{equation}\label{rad-snn-g-a}
  \mathfrak{R}_{S}(\gcal)\leq \frac{R_V}{n}\Big(c\sum_{j=1}^{m}\sum_{i=1}^{n}\gamma^2(\bx_i^\top\bw_j^{(0)})\Big)^{\frac{1}{2}}+
  G_\gamma\sup_{\bW,\bV}\kappa(\bW,\bV)\Big(\frac{2+\sqrt{5}}{n}\big(\sum_{i=1}^{n}\|\bx_i\|_2^2\big)^{\frac{1}{2}}+\frac{c_m\big\|\sum_{i=1}^{n}\bx_i\bx_i^\top\big\|_\sigma^{\frac{1}{2}}}{n}\Big)
\end{equation}
and
  \begin{equation}\label{rad-snn-g-b}
\mathfrak{R}_{S}(\gcal)\leq \frac{R_V}{n}\Big(c\sum_{j=1}^{m}\sum_{i=1}^{n}\gamma^2(\bx_i^\top\bw_j^{(0)})\Big)^{\frac{1}{2}}+
G_\gamma c^{\frac{1}{2}}R_WR_V\Big(\frac{2+\sqrt{5}}{n}\big(\sum_{i=1}^{n}\|\bx_i\|_2^2\big)^{\frac{1}{2}}+\frac{c_m\big\|\sum_{i=1}^{n}\bx_i\bx_i^\top\big\|_\sigma^{\frac{1}{2}}}{n}\Big),
\end{equation}
where $c_m=2\sqrt{2}\big(1+\frac{1}{2\log(2mc)}\big)\log^{\frac{1}{2}}\big(2mc\big\lceil\log_2\big(2R_WR_V(cm)^{\frac{1}{2}}/\sup_{\bW,\bV}\kappa(\bW,\bV)\big)\big\rceil\big)$.
\end{theorem}
\begin{remark}[Comparison]
We make a detailed comparison with two mostly related work~\citep{neyshabur2019towards,magen2023initialization} on initialization-dependent bounds. We will give comparisons with more work on the generalization bounds (Section~\ref{sec:comparison-gen}). The seminal work~\citep{neyshabur2019towards} considered SNNs with the ReLU activation function under a different constraint (we denote $\bv^j$ the $j$-th column of $\bV$)
\[
\gcal_1=\Big\{\bx\mapsto\Psi_{\bW,\bV}(\bx):\bW\in\rbb^{m\times d},\bV\in\rbb^{c\times m}:\|\bv^j\|_2\leq\alpha_j,\|\bw_j-\bw_j^{(0)}\|_2\leq \beta_j,\forall j\in[m]\Big\},
\]
and derived the following Rademacher complexity bound
\begin{equation}\label{neyshabur2019-R}
\mathfrak{R}_{S}(\gcal_1)\leq \frac{2(\sqrt{2c}+2)}{\sqrt{n}}  \|\mathbf{\alpha}\|_2\Big(\|\mathbf{\beta}\|_2\big(\frac{1}{n}\sum_{i=1}^{n}\|\bx_i\|_2^2\big)^{\frac{1}{2}}+
\big(\frac{1}{n}\sum_{i=1}^{n}\|\bW^{(0)}\bx_i\|_2^2\big)^{\frac{1}{2}}\Big),
\end{equation}
where $\mathbf{\alpha}=(\alpha_1,\ldots,\alpha_m)$ and $\mathbf{\beta}=(\beta_1,\ldots,\beta_m)$. The space $\gcal_1$ ignores the correlations among different $\bv^j$ and $\bw_j$ since it imposes component-wise constraints. As a comparison, we impose the Frobenius constraint on $\bV$ and $\bW-\bW^{(0)}$, which preserves the correlation. The Frobenius-norm constraint is a more natural choice due to the implicit regularization methods: gradient methods prefer models close to the initialization and the closeness is often measured by the Frobenius norm~\citep{soltanolkotabi2018theoretical,oymak2020toward,ma2019implicit}. For example, the work~\citep{oymak2020toward,taheri2025sharper,soltanolkotabi2018theoretical} showed that  $\|\bW^{(t)}-\bW^{(0)}\|_F$ is bounded  for iterates $\{\bW^{(t)}\}$ produced by gradient methods applied to shallow and deep neural networks. 


The work~\citep{magen2023initialization} considered the space $\gcal$ with a $L_\gamma$-smooth activation function and $c=1$ (for binary classification problems, it suffices to consider a real-valued prediction function), and derived the following complexity bound
\begin{equation}\label{magen2023-R}
\mathfrak{R}_{S}(\gcal)=\widetilde{O}\Big(\frac{1}{\sqrt{n}}+\frac{R_Vb_x}{\sqrt{n}}\Big(G_\gamma\|\bW^{(0)}\|_\sigma+(G_\gamma+L_\gamma)R_W(1+\|\bW^{(0)}\|_\sigma b_x)\Big)\Big),
\end{equation}
where it was assumed that $b_x=\sup_{\bx}\|\bx\|_2$. This bound is data-independent in the sense of involving $b_x$ instead of $\bx_i$.
If $\gamma(0)=0$ (this is assumed in \citep{magen2023initialization}), then the first term on the right-hand side of Eq.~\eqref{rad-snn-g-a} satisfies
\begin{align*}
\frac{1}{n}\Big(\sum_{j=1}^{m}\sum_{i=1}^{n}\gamma^2(\bx_i^\top\bw_j^{(0)})\Big)^{\frac{1}{2}} & \leq \frac{G_\gamma}{n}\Big(\sum_{j=1}^{m}\sum_{i=1}^{n}\big(\bx_i^\top\bw_j^{(0)}\big)^2\Big)^{\frac{1}{2}}=  \frac{G_\gamma}{n}\Big(\sum_{i=1}^{n}\bx_i^\top\sum_{j=1}^{m}\bw_j^{(0)}(\bw_j^{(0)})^\top\bx_i\Big)^{\frac{1}{2}} \\
& = \frac{G_\gamma}{n}\Big(\sum_{i=1}^{n}\bx_i^\top(\bW^{(0)})^\top\bW^{(0)}\bx_i\Big)^{\frac{1}{2}}=\frac{G_\gamma}{n}\Big(\sum_{i=1}^{n}\|\bW^{(0)}\bx_i\|_2^2\Big)^{\frac{1}{2}},
\end{align*}
which is smaller than the term $\frac{b_x}{\sqrt{n}}G_\gamma\|\bW^{(0)}\|_\sigma$ in Eq.~\eqref{magen2023}. If we consider Kaiming initialization, then $\frac{G_\gamma}{n}\big(\sum_{i=1}^{n}\|\bW^{(0)}\bx_i\|_2^2\big)^{\frac{1}{2}}$ does not increase with $m$.
As shown in Figure~\ref{Fig_W0}, our empirical results support above claim. 


Note the work~\citep{neyshabur2019towards} considered the specific ReLU activation function, while the work~\citep{magen2023initialization} imposed a critical smoothness assumption on the activation function. We extend these discussions to general Lipschitz activation functions.
Moreover, the existing initialization-dependent~\citep{neyshabur2019towards,magen2023initialization} analyses use the Frobenius norm to measure the distance from initialization. In comparison to these results, we also get a bound depending on the path-norm, which is smaller.
\end{remark}
\begin{remark}[Idea and Novelty]
A key challenge in the analysis is to handle the constraint in terms of initialization, due to which the homogeneity property does not apply.
Our basic idea to estimate the Rademacher complexity is to relate $n\mathfrak{R}_{S}(\gcal)$ to the following two terms
\begin{multline*}
 \ebb_{\bm{\sigma}}\sup_{\bW,\bV}\sum_{j=1}^{m}\sum_{k=1}^cv_{kj}\sum_{i=1}^{n}\ibb_{[\|\bw_j-\bw_j^{(0)}\|_2\leq a]}\sigma_{ik}\big(\gamma(\bx_i^\top\bw_j)-\gamma(\bx_i^\top\bw_j^{(0)})\big)\\
  +\ebb_{\bm{\sigma}}\sup_{\bW,\bV}\sum_{j=1}^{m}\sum_{k=1}^cv_{kj}\sum_{i=1}^{n}\ibb_{[\|\bw_j-\bw_j^{(0)}\|_2> a]}\sigma_{ik}\big(\gamma(\bx_i^\top\bw_j)-\gamma(\bx_i^\top\bw_j^{(0)})\big):=I_1+I_2,
\end{multline*}
where $\ibb_{[\cdot]}$ is an indicator function, returning $1$ if the argument holds and $0$ otherwise.
For $I_1$, we use Schwarz's inequality and the constraint $\|\bV\|_F\leq R_V$ to get
\[
I_1\leq R_V\Big(\sum_{j=1}^{m}\sum_{k=1}^c\ebb_{\bm{\sigma}}\sup_{\bw:\|\bw-\bw_j^{(0)}\|_2\leq a}\Big(\sum_{i=1}^{n}\sigma_{ik}\big(\gamma(\bx_i^\top\bw)-\gamma(\bx_i^\top\bw_j^{(0)})\big)\Big)^2\Big)^{\frac{1}{2}}.
\]
The term $\ebb_{\bm{\sigma}}\sup_{\bw:\|\bw-\bw_j^{(0)}\|_2\leq a}\big(\sum_{i=1}^{n}\sigma_{ik}\big(\gamma(\bx_i^\top\bw)-\gamma(\bx_i^\top\bw_j^{(0)})\big)\big)^2$ is a second-moment of Rademacher complexity, and we control it by the  Efron-Stein inequality.

Our idea to control $I_2$ is to rewrite it as follows
\begin{align*}
  & I_2 = \ebb_{\bm{\sigma}}\sup_{\bW,\bV}\sum_{j=1}^{m}\sum_{k=1}^cv_{kj}\|\bw_j-\bw_j^{(0)}\|_2\ibb_{[\|\bw_j-\bw_j^{(0)}\|_2> a]}\sum_{i=1}^{n}\sigma_{ik}\big(\gamma(\bx_i^\top\bw_j)-\gamma(\bx_i^\top\bw_j^{(0)})\big)/\|\bw_j-\bw_j^{(0)}\|_2 \\
  & \leq \ebb_{\bm{\sigma}}\sup_{\bW,\bV}\sum_{j=1}^{m}\sum_{k=1}^c|v_{kj}|\|\bw_j-\bw_j^{(0)}\|_2\max_{j\in[m]}\max_{k\in[c]}\ibb_{[\|\bw_j-\bw_j^{(0)}\|_2> a]}\Big|\sum_{i=1}^{n}\sigma_{ik}\big(\gamma(\bx_i^\top\bw_j)-\gamma(\bx_i^\top\bw_j^{(0)})\big)/\|\bw_j-\bw_j^{(0)}\|_2\Big|.
\end{align*}
Note $\sum_{j=1}^{m}\sum_{k=1}^c\big|v_{kj}\big|\|\bw_j-\bw_j^{(0)}\|_2$ is the path-norm and then it suffices to estimate (we assume $c=1$ here for simplicity)
\[
I_3:=\ebb_{\bm{\sigma}}\sup_{\bW}\max_{j\in[m]}\ibb_{[\|\bw_j-\bw_j^{(0)}\|_2> a]}\big|\sum_{i=1}^{n}\sigma_{ik}\big(\gamma(\bx_i^\top\bw_j)-\gamma(\bx_i^\top\bw_j^{(0)})\big)/\|\bw_j-\bw_j^{(0)}\|_2\big|,
\]
which is achieved by a peeling trick. Note that $I_3$ involves a supremum over $\bW\in\wcal$ and a maximum over $j\in[m]$. Our key idea to handle this supremum and maximum is to relate it to the Rademacher complexity of a union of function spaces. Specifically, we define $r_k=2^{k-1}a$ for $k\in[K]$ with  $K:=1+\lceil\log_2(R_W/a)\rceil$, and introduce
\[
\hcal_{k,j}=\Big\{\bx\mapsto \big(\gamma(\bx^\top\bw)-\gamma(\bx^\top\bw_j^{(0)})\big)/\|\bw-\bw_j^{(0)}\|_2:\bw\in\rbb^{d},
r_k<\|\bw-\bw_j^{(0)}\|_2\leq r_{k+1}\Big\},\; k\in[K],j\in[m].
\]
Then, we can control $I_3$ by the Rademacher complexity of the union of $\hcal_{k,j}$. 
We show $\mathfrak{R}_S\big(\hcal_{k,j}\big)\leq \frac{2G_\gamma }{n}\big(\sum_{i=1}^{n}\|\bx_i\|_2^2\big)^{\frac{1}{2}}$, which is then used to control $I_3$ by existing Rademacher complexity bounds of a union of function spaces~\citep{maurer2014inequality}. The construction of $\hcal_{k,j}$ depends on the initialization, which is a key to handle the initialization-dependent constraint.

We combine the bounds on $I_1$ and $I_2$ to get complexity bounds in Theorem~\ref{thm:rad-snn-g}. The detailed proof is given in Section~\ref{sec:proof-rad-snn-g}.
\end{remark}

The following theorem gives lower bounds on Rademacher complexities. For simplicity, we consider binary classification with $c=1$ and the ReLU activation function. This shows that our upper bounds in Theorem~\ref{thm:rad-snn-g} are tight up to a logarithmic factor. The proof is given in Section~\ref{sec:proof-rad-snn-g}.
Our basic idea is to relate the complexity of SNNs to the complexity of function spaces with only a single neuron, the latter of which can be further related to the complexity of linear function spaces by the simple observation that $t=\gamma(t)-\gamma(-t)$ if $\gamma$ is the ReLU activation function.
\begin{theorem}[Lower Bounds\label{thm:lower}]
Let $r_0=\min_{j\in[m]}\|\bw^{(0)}_j\|_2$, $c=1$ and $\gamma$ be the ReLU activation. If $R_W\geq r_0$ and $\gcal$ is defined in Eq.~\eqref{gcal}, then
\[
\mathfrak{R}_S(\gcal)\geq \frac{(R_W-r_0)R_V}{4\sqrt{2}n}\big(\sum_{i=1}^{n}\|\bx_i\|_2^2\big)^{\frac{1}{2}}+\frac{R_V}{2\sqrt{2}n}\big(\sum_{i=1}^{n}\sum^m_{j=1}\gamma^2(\bx_i^\top\bw_j^{(0)})\big)^{\frac{1}{2}}
\]
\end{theorem}
\begin{remark}
  Lower bounds of Rademacher complexity for neural networks have been studied in the literature~\citep{neyshabur2019towards,bartlett2017spectrally}, which, however, focused on the case with $\bW^{(0)}=0$ and different hypothesis spaces.
  For example, if $\bW^{(0)}=0$, it was shown that
  \[
  \mathfrak{R}_S(\gcal_1)\gtrsim \frac{\sum_{j=1}^{m}\alpha_j\beta_j}{n}\big(\sum_{i=1}^{n}\|\bx_i\|_2^2\big)^{\frac{1}{2}}.
  \]
  The lower bounds in \citep{bartlett2017spectrally} were developed for $\bW^{(0)}=0$ and constraints expressed by the spectral norm.
  We extend these discussions to lower bounds on initialization-dependent hypothesis spaces with constraints expressed by the Frobenius norm. The condition $R_W\geq r_0$ is mild as $r_0$ refers to a single neuron with the smallest norm.
\end{remark}

Our Rademacher complexity bounds directly imply generalization bounds, which are expressed in terms of the Frobenius norm of $\bX:=(\bx_1,\ldots,\bx_n)\in\rbb^{d\times n}$ and the path-norm of the distance from initialization. The proof of Theorem~\ref{thm:generalization} is given in Section~\ref{sec:proof-thm-generalization}.
We impose a Lipschitzness assumption on the loss function. Popular Lipschitz loss functions include the multinomial logistic loss $\ell_y(\bm{t})=\log\big(\sum_{k\in[c]}\exp(t_j-t_y)\big)$ and the multi-class margin-based loss~\citep{lei2023generalization}. We hide logarithmic factors in $\widetilde{O}(\cdot)$.

\begin{definition}[Lipschitzness]
  Let $G>0$.
  We say a vector-valued function $g:\rbb^c\mapsto \rbb$ is $G$-Lipschitz continuous if
  \[
  \big|g(t_1,\ldots,t_c)-g(t_1',\ldots,t_c')\big|\leq G\big(\sum_{j=1}^{c}(t_j-t_j')^2\big)^{\frac{1}{2}}.
  \]
\end{definition}

\begin{theorem}[Generalization Bounds\label{thm:generalization}]
  Assume $\ell_y(\cdot)$ is $G$-Lipschitz continuous and $\ell_y(\Psi_{\bW,\bV}(\bx))\in[0,b]$ almost surely.  Then, with probability at least $1-\delta$ the following inequality holds uniformly for all $(\bW,\bV)$
  \[
  F(\bW,\bV)-F_S(\bW,\bV)= \widetilde{O}\Big(\frac{G\|\bV\|_F}{n}\Big(c\sum_{j=1}^{m}\sum_{i=1}^{n}\gamma^2(\bx_i^\top\bw_j^{(0)})\Big)^{\frac{1}{2}}+
  \frac{GG_\gamma \kappa(\bW,\bV)\|\bX\|_F}{n}+b\Big(\frac{\log(1/\delta)}{n}\Big)^{\frac{1}{2}}\Big).
  \]
\end{theorem}
\begin{table}
  \centering\renewcommand{\arraystretch}{1.36}
  \begin{tabular}{|c|c|c|c|c|}
    \hline
    \# & Reference & Activation & Bound & D \\ \hline
    (1) & \citep{bartlett2019nearly} & Piecewise linear & $\widetilde{O}(\sqrt{dm})$ & $-$\\ \hline
    (2) & \citep{bartlett2002rademacher} & Lipschitz & $\widetilde{O}\big(\|\bW\|_{\infty,1}\|\bV\|_{\infty,1}\big)$ & \checkmark\\ \hline
    (3) & \citep{neyshabur2015norm} & ReLU & $\widetilde{O}\big(\sum_{j=1}^{m}|v_{j}|\|\bw_j\|_2\big)$ & $-$\\ \hline
    (4) & \citep{golowich2018size} & ReLU & $\widetilde{O}\big(\|\bW\|_F\|\bV\|_F\big)$ & \checkmark\\ \hline
    (5) & \citep{bartlett2017spectrally} & Lipschitz & $\widetilde{O}\big(\|\bW\|_\sigma\|\bV-\bV^{(0)}\|_{1,2}+\|\bW-\bW^{(0)}\|_{1,2}\|\bV\|_\sigma\big)$ & \checkmark\\ \hline
    (6) & \citep{neyshabur2018pac} & ReLU & $\widetilde{O}\big(\|\bW\|_\sigma\|\bV-\bV^{(0)}\|_F+\sqrt{m}\|\bW-\bW^{(0)}\|_F\|\bV\|_\sigma\big)$ & $-$ \\ \hline
    (7) & \citep{neyshabur2019towards} & ReLU & $\widetilde{O}\big(\|\bW^{(0)}\|_\sigma\|\bV\|_F+\|\bW-\bW^{(0)}\|_F\|\bV\|_F+\sqrt{m}\big)$ & \checkmark \\ \hline
    (8) & \citep{magen2023initialization} & Lipschitz, smooth & $\widetilde{O}\big(\frac{1}{b_x}+\|\bV\|_F\big(\|\bW^{(0)}\|_\sigma+R_W(1+\|\bW^{(0)}\|_\sigma \sup_{\bx}\|\bx\|_2)\big)\big)$ & $-$\\ \hline
    (9) & \citep{daniely2024sample} & Lipschitz & $\widetilde{O}\big(\|\bW^{(0)}\|_\sigma\|\bV\|_F+\|\bW-\bW^{(0)}\|_F\|\bV\|_F\big)$ & $-$\\ \hline
    (10) & Thm~\ref{thm:generalization} & Lipschitz & $\widetilde{O}\big(\frac{\|\bV\|_F}{\|\bX\|_F}\big(\sum_{j=1}^{m}\sum_{i=1}^{n}\gamma^2(\bx_i^\top\bw_j^{(0)})\big)^{\frac{1}{2}}+\sum_{j=1}^{m}|v_{j}|\|\bw_j-\bw_j^{(0)}\|_2\big)$ & \checkmark\\
    \hline
  \end{tabular}
  \caption{Comparison with existing generalization bounds for two-layer networks with constant number of outputs. For the column `D', `\checkmark' means data-dependent bounds, while `$-$' means data-independent bounds.
  We ignore a factor of ${\|\bX\|_F}/{n}$ for data-dependent bounds, and a factor of $\sup_{\bx}\|\bx\|_2/\sqrt{n}$ for data-independent bounds (note ${\|\bX\|_F}/{n}\leq \sup_{\bx}\|\bx\|_2/\sqrt{n}$). For simplicity, we assume $c=1$ here. }\label{tab:comp-gen}
\end{table} 
\subsection{Comparison with Existing Generalization Bounds\label{sec:comparison-gen}}
We now compare our generalization bounds with existing results. For simplicity, we ignore the constants $G$ and $G_\gamma$, and also the term $b\big(\frac{\log(1/\delta)}{n}\big)^{\frac{1}{2}}$ here. Then, our bound in Theorem~\ref{thm:generalization} takes the following simple form
\begin{equation}\label{gen-ours}
  F(\bW,\bV)-F_S(\bW,\bV)=\widetilde{O}\Big(\frac{\|\bV\|_F}{n}\Big(c\sum_{j=1}^{m}\sum_{i=1}^{n}\gamma^2(\bx_i^\top\bw_j^{(0)})\Big)^{\frac{1}{2}}+\frac{\kappa(\bW,\bV)\|\bX\|_F}{n}\Big).
\end{equation}
The work~\citep{bartlett2019nearly,bartlett2002rademacher,neyshabur2015norm,golowich2018size} considered initialization-independent hypothesis spaces. Specifically, the paper~\citep{bartlett2019nearly} studied the VC dimension of neural networks with piecewise linear activation functions and derived bounds of order $\widetilde{O}(\sup_{\bx}\|\bx\|_2\sqrt{dm}/\sqrt{n})$. The work~\citep{bartlett2002rademacher,golowich2018size} derived norm-based capacity bounds for neural networks, where the bounds depend on the product of the norm of $\bW$ and $\bV$. The work~\citep{neyshabur2015norm} introduced the standard path-norm, and derived complexity bounds in terms of the path-norm with $\bW^{(0)}=0$. As illustrated in~\citep{neyshabur2019towards,dziugaite2017computing}, the norm of these matrices can be significantly larger than the distance from the initialization matrix. The proof techniques used in~\citep{bartlett2002rademacher,neyshabur2015norm,golowich2018size} do not allow for getting bounds with norms measured from initialization. For example, the analysis in \citep{golowich2018size} crucially relies on the positive homogeneity of the ReLU function, and one cannot use this homogeneity property if we consider either a general Lipschitz activation function or initialization-dependent spaces. Therefore, one has to introduce new techniques to get tighter bounds depending on the distance from initialization.

The work~\citep{bartlett2017spectrally} used the Maurey sparsification lemma to derive the following generalization bounds
\begin{equation}\label{bartlett2017}
  F(\bW,\bV)-F_S(\bW,\bV)=\widetilde{O}\Big(\frac{\|\bX\|_F}{n}\Big(\|\bW\|_\sigma\|\bV-\bV^{(0)}\|_{1,2}+\|\bW-\bW^{(0)}\|_{1,2}\|\bV\|_\sigma\Big)\Big).
\end{equation}
A key term in this bound is $\|\bW-\bW^{(0)}\|_{1,2}\|\bV\|_\sigma$, which is replaced by $\|\bW-\bW^{(0)}\|_F\|\bV\|_F$ in Eq.~\eqref{gen-ours}. Note that $\|\bV\|_\sigma$ and $\|\bV\|_F$ differs at most by a factor of $\sqrt{c}$, which is a small constant. On the other hand, $\|\bW-\bW^{(0)}\|_{1,2}$ can be larger than $\|\bW-\bW^{(0)}\|_F$ by a factor of $\sqrt{m}$. Therefore, it admits a square-root dependency on $m$ and yields a loose bounds when $m$ is large. The work~\citep{neyshabur2018pac} took a PAC-Bayes approach and the associated bound involves $\sqrt{m}\|\bW-\bW^{(0)}\|_F\|\bV\|_\sigma$, which again admits a square-root dependency on $m$. The work~\citep{neyshabur2019towards} introduced a new decomposition to control the Rademacher complexity of SNNs with the ReLU activation, and showed
\begin{equation}\label{neyshabur2019}
  F(\bW,\bV)-F_S(\bW,\bV)=\widetilde{O}\Big(\frac{\|\bX\|_F}{n}\Big(\|\bW^{(0)}\|_\sigma\|\bV\|_F+\sqrt{c}\|\bW-\bW^{(0)}\|_F\|\bV\|_F+\sqrt{m}\Big)\Big).
\end{equation}
It is clear that our bound in Eq.~\eqref{gen-ours} improves Eq.~\eqref{neyshabur2019} by removing the term $\|\bW^{(0)}\|_\sigma\|\bV\|_F+\sqrt{m}$, and replacing the term $\|\bW-\bW^{(0)}\|_F\|\bV\|_F$ with $\kappa(\bW,\bV)$.
Furthermore, by Schwarz's inequality, we have
\begin{align}\label{path-fro}
  \sum_{j=1}^{m}\sum_{k=1}^c|v_{kj}|\|\bw_j-\bw_j^{(0)}\|_2 & \leq \Big(\sum_{j=1}^{m}\big(\sum_{k=1}^c|v_{kj}|\big)^2\Big)^{\frac{1}{2}}\Big(\sum_{j=1}^{m}\|\bw_j-\bw_j^{(0)}\|_2^2\Big)^{\frac{1}{2}}\notag\\
  & \leq \Big(c\sum_{j=1}^{m}\sum_{k=1}^cv^2_{kj}\Big)^{\frac{1}{2}}\|\bW-\bW^{(0)}\|_F =c^{\frac{1}{2}}\|\bW-\bW^{(0)}\|_F\|\bV\|_F.
\end{align}
Therefore, the path-norm is always smaller than the product of the Frobenius norms and  $\sqrt{c}$. In the extreme case, we can have $\kappa(\bW,\bV)\ll c^{\frac{1}{2}}\|\bW-\bW^{(0)}\|_F\|\bV\|_F$. For example, consider $c=1$, $\bV=(1/\sqrt{m},\ldots,1/\sqrt{m})$, $\bw_j-\bw_j^{(0)}=0$ if $j\neq 1$. Then, we have $\kappa(\bW,\bV)=\frac{1}{\sqrt{m}}\|\bw_1-\bw_1^{(0)}\|_2$ and $\|\bW-\bW^{(0)}\|_F\|\bV\|_F=\|\bw_1-\bw_1^{(0)}\|_2$. Therefore, there is a gap by a factor of $\sqrt{m}$ between $\kappa(\bW,\bV)$ and $\|\bW-\bW^{(0)}\|_F\|\bV\|_F$.
Finally, the work~\citep{neyshabur2019towards} only considered the ReLU activation function, which is extended to general Lipschitz activation functions here.

The more recent work~\citep{magen2023initialization} considered Lipschitz and smooth activation functions with $\gamma(0)=0$ and $c=1$, and got
\begin{equation}\label{magen2023}
  F(\bW,\bV)-F_S(\bW,\bV)=\widetilde{O}\Big(\frac{1}{\sqrt{n}}+\frac{\|\bV\|_F\sup_{\bx}\|\bx\|_2}{\sqrt{n}}\Big(G_\gamma\|\bW^{(0)}\|_\sigma+(G_\gamma+L_\gamma)R_W(1+\|\bW^{(0)}\|_\sigma \sup_{\bx}\|\bx\|_2)\Big)\Big).
\end{equation}
First, the analysis in \citep{magen2023initialization} applies to smooth activation functions, which does not apply to the ReLU activation function. We remove the smoothness assumption. Second, the bound in Eq.~\eqref{magen2023} is not data-dependent in the sense of involving the factor of $\sup_{\bx}\|\bx\|_2/\sqrt{n}$, which is replaced by a data-dependent term $\|\bX\|_F/n$ in Eq.~\eqref{gen-ours}.


The work~\citep{daniely2024sample} used the approximate description length to develop generalization bounds for $c=1$
\begin{equation}\label{daniely2024}
  F(\bW,\bV)-F_S(\bW,\bV)=\widetilde{O}\Big(\frac{G_\gamma\sup_{\bx}\|\bx\|_2}{\sqrt{n}}\Big(\|\bW^{(0)}\|_\sigma\|\bV\|_F+\|\bW-\bW^{(0)}\|_F\|\bV\|_F\Big)\Big).
\end{equation}
Again, this bound is data-independent and depends on the product of the Frobenius norm. We replace the term $\|\bW^{(0)}\|_\sigma\|\bV\|_F/\sqrt{n}$ in Eq.~\eqref{daniely2024} by $\|\bV\|_F\big(\sum_{j=1}^{m}\sum_{i=1}^{n}\gamma^2(\bx_i^\top\bw_j^{(0)})\big)^{\frac{1}{2}}/n$, and replace the Frobenius norm by the path-norm, which can be much smaller. Finally, the work~\citep{daniely2024sample} studied generalization by approximation description length instead of Rademacher complexity analysis. Their bound involves various logarithmic and constant factors that are difficult to track, and therefore their bound is difficult to compute in empirical analyses. As a comparison, our generalization bounds are based on Rademacher complexity analysis of function spaces, and our bounds are directly computable in practice.
We summarize the comparison in Table~\ref{tab:comp-gen}.

\section{Empirical Studies\label{sec:empirical}}
In this section, we conduct empirical evaluations on binary classification tasks on the MNIST ($d=784,\ n=60000$) and CIFAR 10 ($d=1024,\ n=50000$) datasets to consolidate our theoretical studies.

\subsection{Experimental Setup}
For MNIST, we resize the image from $28\times28$ to $32\times32$ and consider a binary classification problem of identifying whether a handwritten digit is $1$  or $7$. We have $6742$ images with label $1$, and $6265$ images with label $7$, resulting a binary classification problem with sample size $n=6742+6265=13007$. For the CIFAR-10 dataset, we classify whether an object is an airplane or an automobile. There are $5000$ images of each in the training dataset. We train two-layer ReLU neural networks of varying widths, where the number of hidden units $m$ ranges from $2^6$ to $2^{18}$. For CIFAR 10, $m$ is restricted to $2^6$-$2^{17}$ due to the limited computational resources. The network is trained using the binary cross-entropy with logits loss, defined as
\begin{equation}
\mathcal{L}_{\mathrm{BCE}}(\bW,\bV)
= - \frac{1}{n} \sum_{i=1}^{n} \Big[ y_i \log \sigma(\Psi_{\bW,\bV}(\bx_i)) + (1 - y_i) \log(1 - \sigma(\Psi_{\bW,\bV}(\bx_i))) \Big],
\end{equation}
where $y_i\in\{0,1\}$ represents the true labels, $\Psi_{\bW,\bV}(\bx_i)$ denotes the network output, and $\sigma(z) = 1 / (1 + e^{-z})$ is the sigmoid function. We optimize this loss using stochastic gradient descent (SGD) with a mini-batch size of $256$, momentum of $0.9$, and an initial learning rate of $0.001$. Each model is trained until the training error is less than $0.1$ or the number of epochs reaches $20$ for MNIST. The maximum epoch for CIFAR 10 is set to $50$. We initialized both layers with Kaiming Gaussian distribution. We have flatten the input into a one-dimensional vector and apply $\ell_2$-normalization, ensuring the Euclidean norm of the input equals $1$. Considering the generalization bounds of the trained model, we map the labels from $\{0,1\}$ to $\{-1,1\}$ and choose the $1$-Lipschitz loss as
\begin{equation*}
    \ell_y(\Psi_{\bW,\bV}(\bx)) =
    \begin{cases}
        0 \quad  &\text{if } y\Psi_{\bW,\bV}(\bx) > 1 \\
        1-y\Psi_{\bW,\bV}(\bx) \quad &\text{if } y\Psi_{\bW,\bV}(\bx) \in[0,1] \\
        1 \quad & \text{if } y\Psi_{\bW,\bV}(\bx) < 0,
    \end{cases}
\end{equation*}
where $y\in\{-1,1\}$ denotes the true labels. The experiments were completed using the PyTorch framework on NVIDIA RTX4060 GPUs. Five trials were conducted with different random seeds to eliminate interference.

\subsection{Observations and Comparison}
\begin{figure}[htbp]
    \centering
    \begin{subfigure}{0.45\textwidth}
    \includegraphics[width=\textwidth]{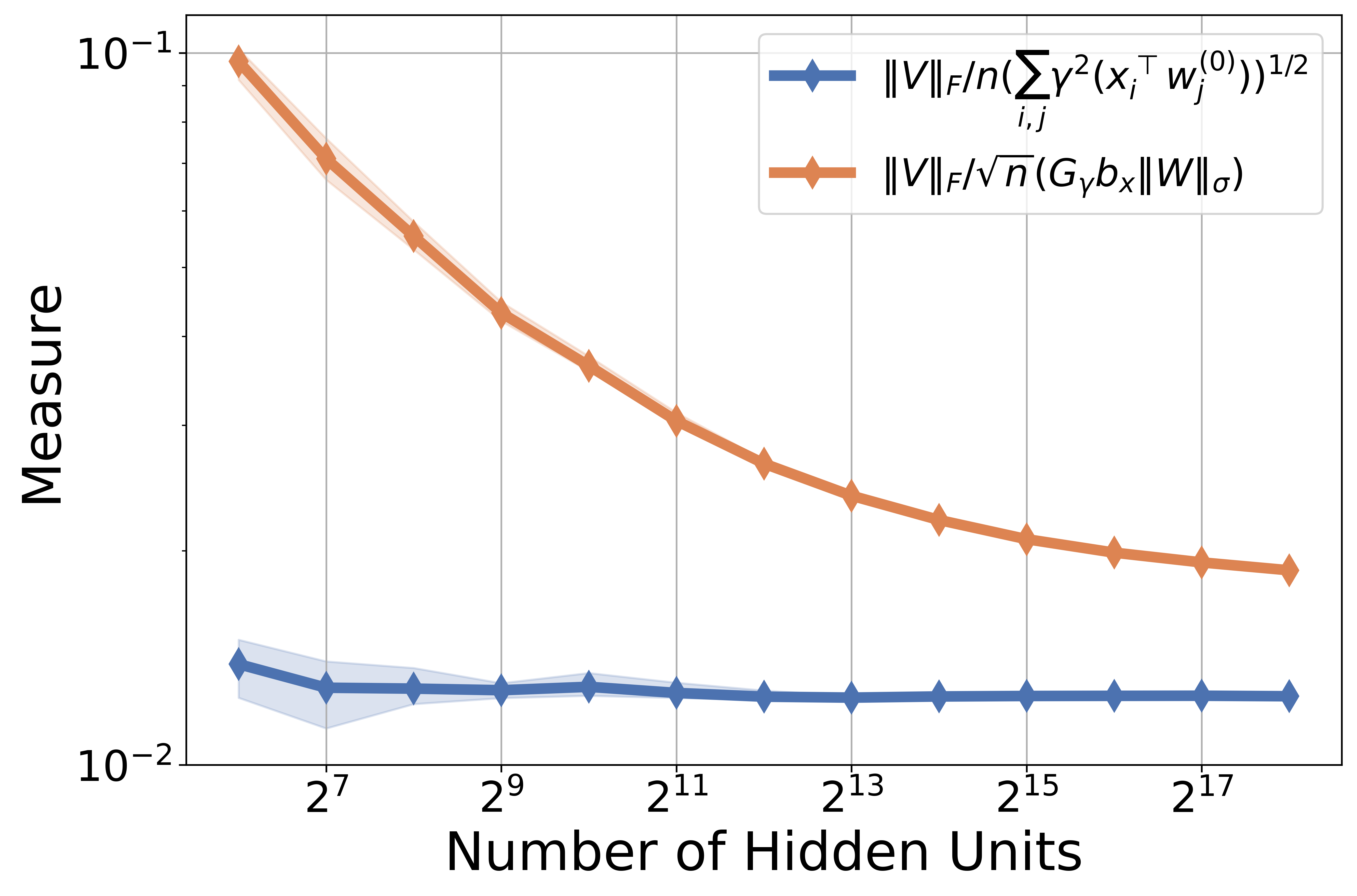}
    \caption{$\frac{\|V\|_F}{n}\big(\sum_{j=1}^{m}\sum_{i=1}^{n}\gamma^2(\bx_i^\top\bw_j^{(0)})\big)^{\frac{1}{2}}$ versus $\|V\|_F(G_\gamma b_x\|\bW^{(0)}\|_\sigma)/\sqrt{n}$.}
    \label{Fig_W0}
    \end{subfigure}
    \hfill
    \begin{subfigure}{0.45\textwidth}
        \includegraphics[width=\textwidth]{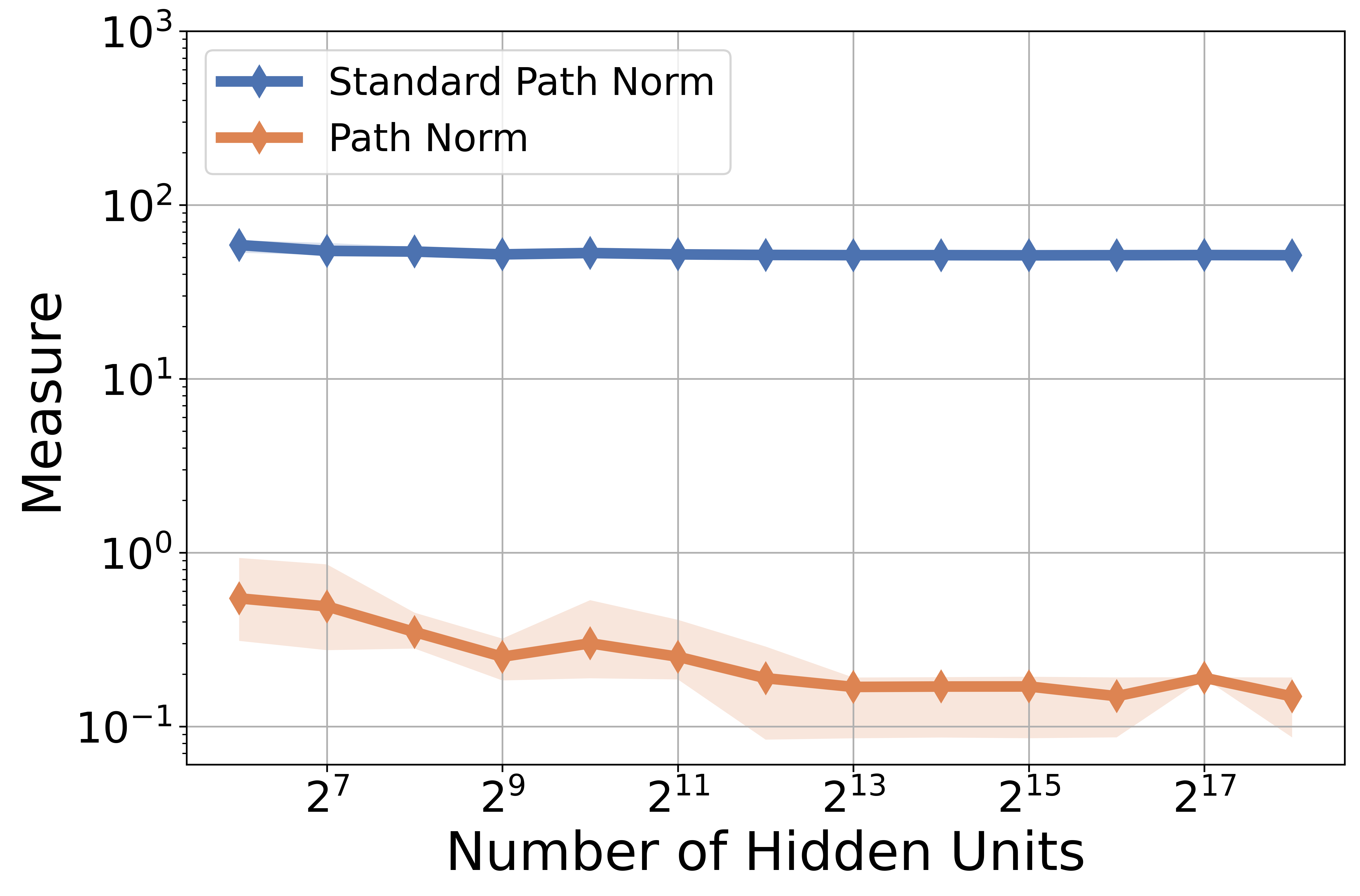}
    \caption{A comparison between standard path norm and path norm (Definition~\ref{def:path}).}
    \label{Fig_path_norm}
    \end{subfigure}
    \caption{Behavior of $\frac{\|V\|_F}{n}\big(\sum_{j=1}^{m}\sum_{i=1}^{n}\gamma^2(\bx_i^\top\bw_j^{(0)})\big)^{\frac{1}{2}}$, $\|V\|_F(G_\gamma b_x\|\bW^{(0)}\|_\sigma)/\sqrt{n}$, and path norm with respect to the number of hidden units $m$ on the MNIST dataset. The gray area represents the value range under different random trials.}
    \label{Fig_W0_path_norm}
\end{figure}

Figure~\ref{Fig_W0} compares $\frac{\|V\|_F}{n}\big(\sum_{j=1}^{m}\sum_{i=1}^{n}\gamma^2(\bx_i^\top\bw_j^{(0)})\big)^{\frac{1}{2}}$ and $\|V\|_F(G_\gamma b_x\|\bW^{(0)}\|_\sigma)/\sqrt{n}$, versus the number of hidden units $m$. We initialized the model weights using Kaiming Gaussian distribution. Empirically, we observe that $\frac{\|V\|_F}{n}\big(\sum_{j=1}^{m}\sum_{i=1}^{n}\gamma^2(\bx_i^\top\bw_j^{(0)})\big)^{\frac{1}{2}}$ is consistently smaller than $\|V\|_F(G_\gamma b_x\|\bW^{(0)}\|_\sigma)/\sqrt{n}$, and the difference is significant for small $m$. With Kaiming Gaussian initialization, Figure~\ref{Fig_path_norm} compares the evolution of standard path-norm (i.e., $\bW^{(0)}=0$)~\citep{neyshabur2015path,neyshabur2015norm,yang2024nonparametric} and the path-norm (Definition~\ref{def:path}) over SNNs of increasing sizes. This observation demonstrates that our path-norm is significantly smaller than the standard path-norm and insensitive to the model width, showing that the distance from initialization can be substantially smaller than the norm of the model. This justifies the effectiveness of incorporating the distance from initialization in the generalization analysis.


\begin{figure}[htbp]
    \centering
    \begin{subfigure}{0.45\textwidth}
        \includegraphics[width=\textwidth]{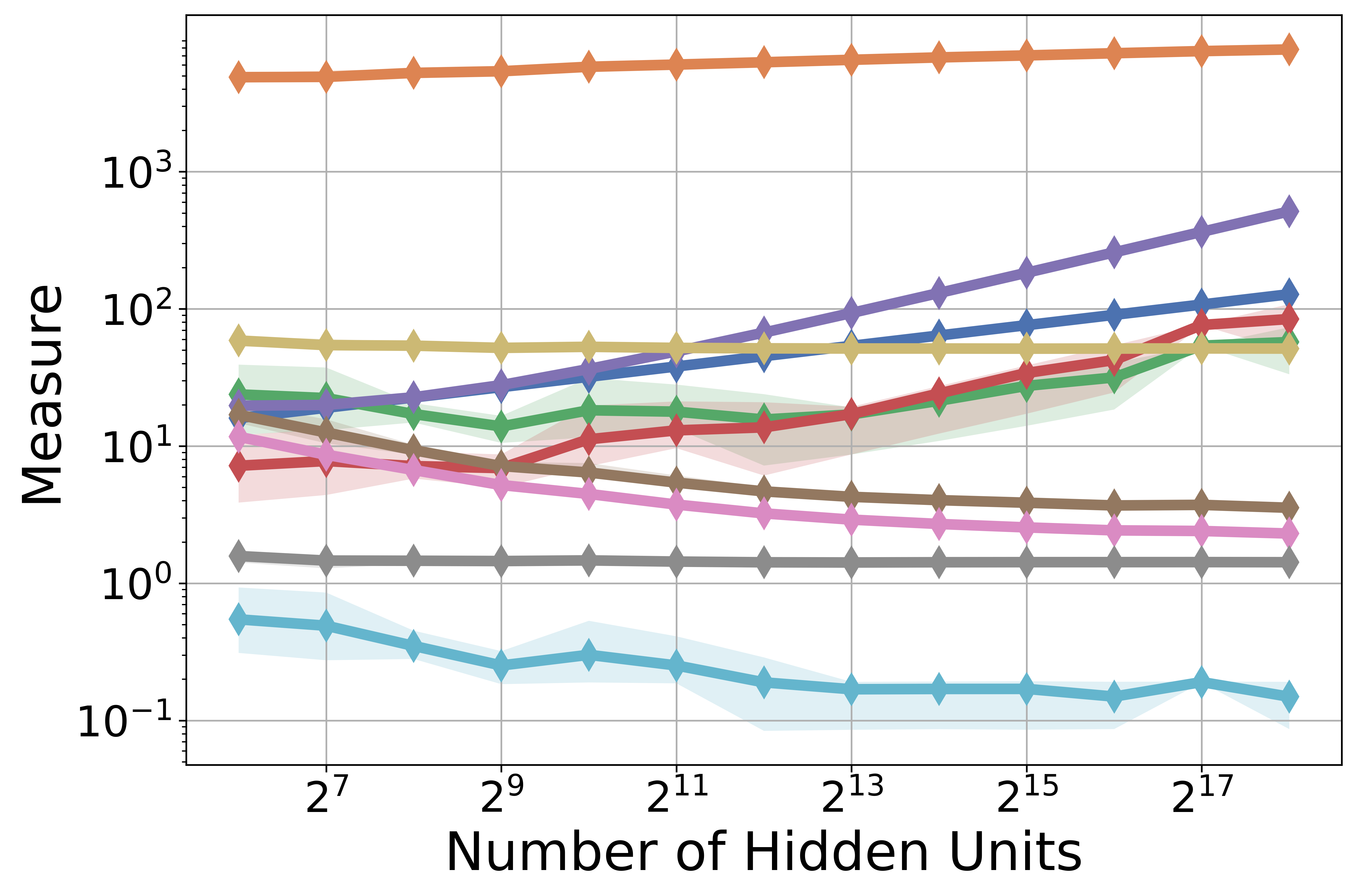}
    \end{subfigure}
    \begin{subfigure}{0.45\textwidth}
        \includegraphics[width=\textwidth]{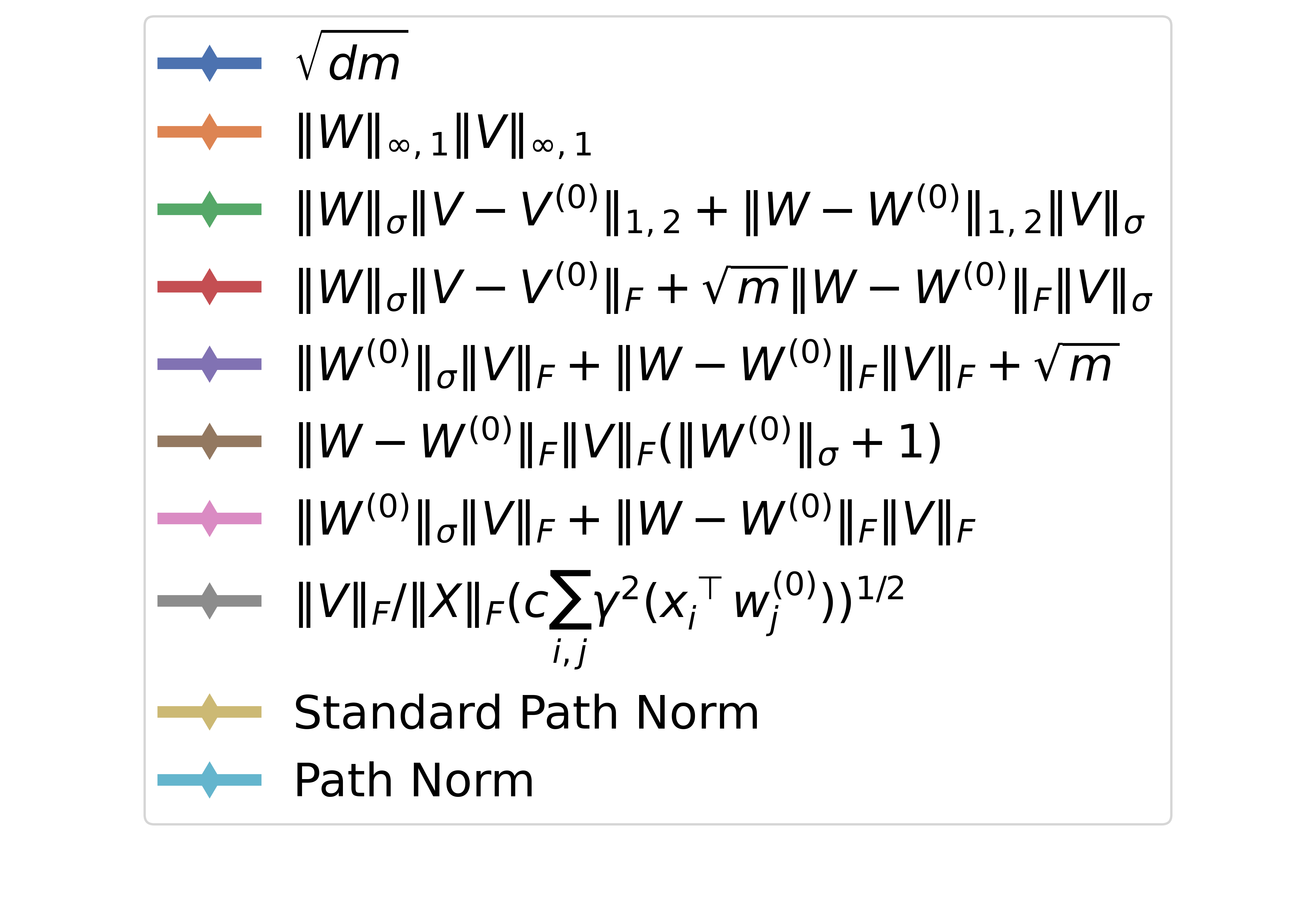}
    \end{subfigure}
    \caption{Dominant terms in the generalization bounds summarized in Table~\ref{tab:comp-gen} on MNIST dataset. The gray area represents the value range under different random trials.}
    \label{Fig_Measure_comparison}
\end{figure}

Figure~\ref{Fig_Measure_comparison} compares dominant terms in the generalization bounds as stated in Table~\ref{tab:comp-gen}. It shows that the path norm in Definition~\ref{def:path} remains substantially smaller and less sensitive to the growth of $m$ compared to other complexity measures used in the existing generalization analyses. The results illustrate the advantage of incorporating the initialization-dependent path-norm in the generalization analysis, which consequently explain the improved generalization guarantees for SNNs.

\begin{figure}[htbp]
    \centering
    \begin{subfigure}{0.49\textwidth}
        \includegraphics[width=\textwidth]{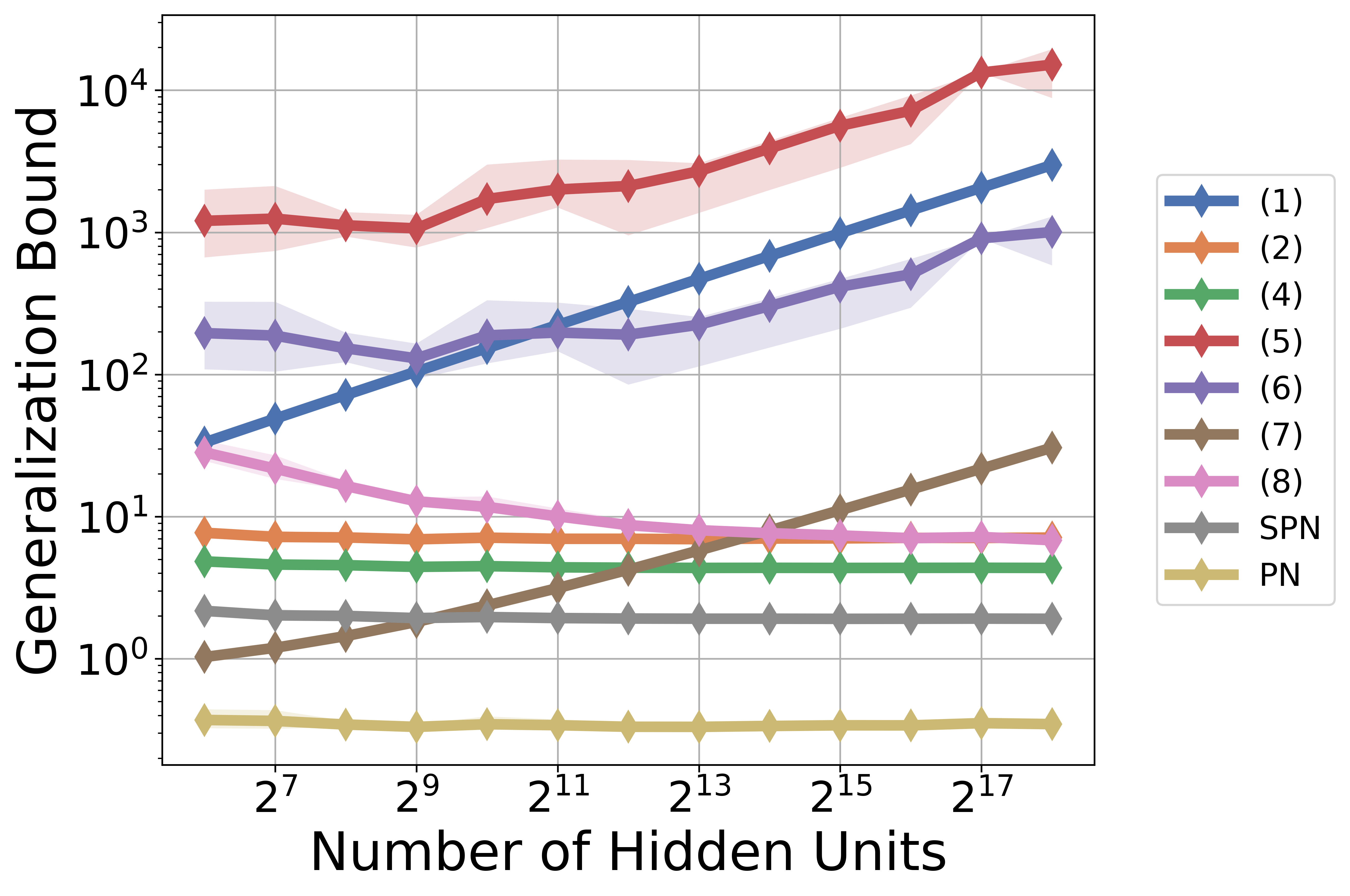}
    \caption{MNIST}
    \end{subfigure}
    \hfill
    \begin{subfigure}{0.49\textwidth}
        \includegraphics[width=\textwidth]{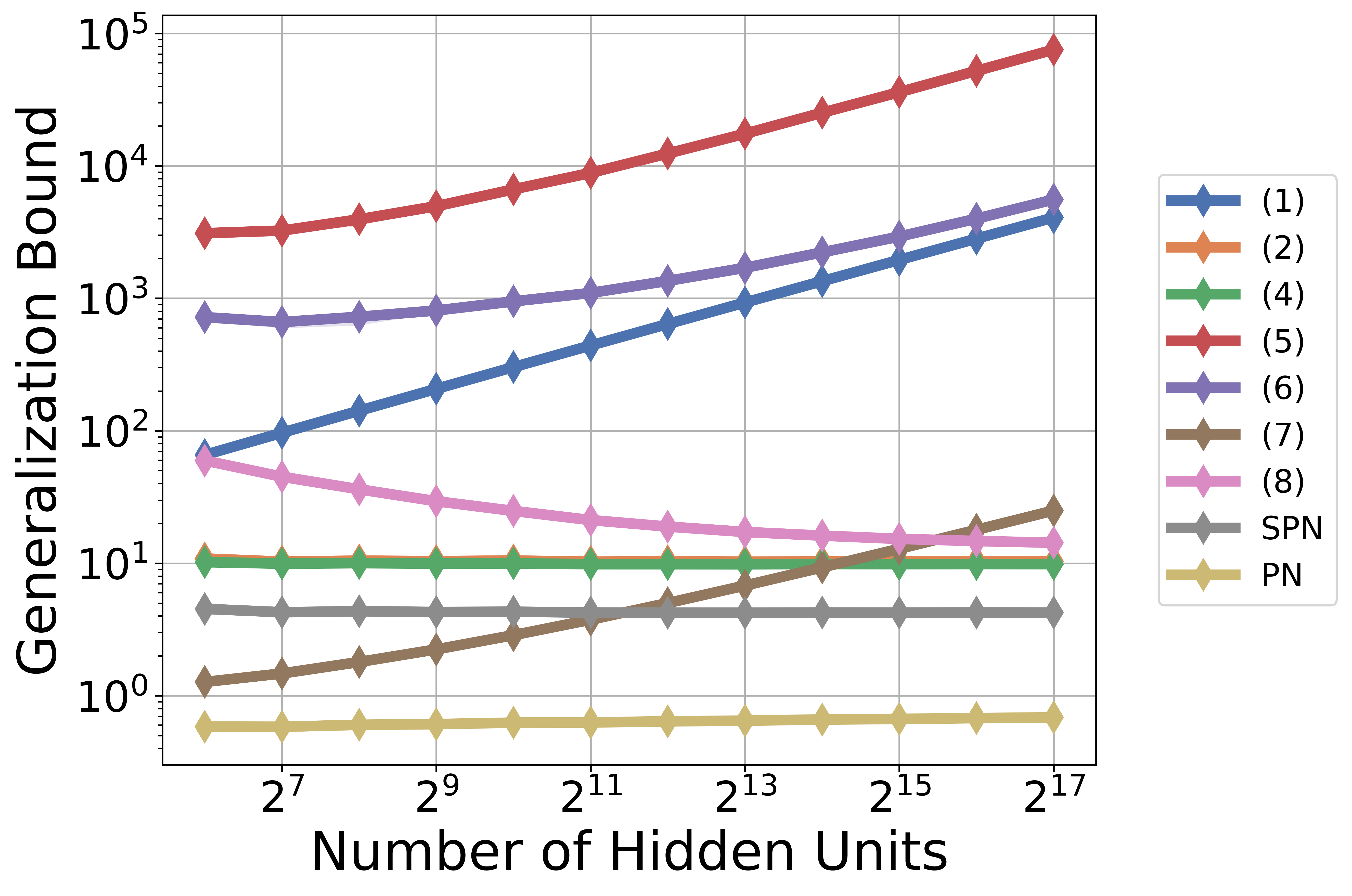}
    \caption{CIFAR10}
    \end{subfigure}
    \caption{The comparison of existing generalization bounds with all the terms and constants across the MNIST and Cifar10 datasets. The gray area represents the value range under different random trials. The label corresponds to the index number in Table~\ref{tab:comp-gen}. SPN denotes generalization bounds based on standard path norm, i,e., Eq.~\eqref{spn}, and PN means our generalization bounds based on path norm, i.e., Eq.~\eqref{constant-bound}.}
    \label{Fig_Com}
\end{figure}

To further show the strength of our generalization bounds, we present a comparison of our generalization bounds with existing ones by incorporating all relevant terms and constants. In particular, we use Eq.~\eqref{constant-bound} as our exact generalization bound\footnote{{The constant $\sqrt{2}$ in the first term of the right-hand side of Eq.~\eqref{constant-bound} can be removed since we consider binary classification, i.e., $c=1$, and the bound in Lemma~\ref{lem:maurer-contraction} can be improved by a factor of $1/\sqrt{2}$.}}, and compare it with the exact bounds in \citep{bartlett2019nearly,bartlett2002rademacher, golowich2018size, bartlett2017spectrally, neyshabur2018pac, neyshabur2019towards, magen2023initialization}. {Note that \citet{magen2023initialization} assume a smooth activation function to derive generalization bounds. To facilitate comparison, we omit the smoothness parameter in their bounds.} For the standard path-norm, we use the following generalization bound as a direct corollary of Eq.~\eqref{rad-snn-spn} and the peeling arguments
\begin{align}
&F(\bW,\bV)-F_S(\bW,\bV)\leq \frac{4}{n}\big(\kappa_s(\bW,\bV)+1\big)\big(\sum_{i=1}^{n}\|\bx_i\|_2^2\big)^{\frac{1}{2}}+
3\Big(\frac{\log(2(\kappa_s(\bW,\bV)\!+\!1)(\kappa_s(\bW,\bV)\!+\!2)/\delta)}{2n}\Big)^{\frac{1}{2}},\label{spn}
\end{align}
where $\kappa_s(\bW,\bV)=\sum_{j=1}^{m}|v_j|\|\bw_j\|_2$ is the standard path-norm.
The experimental results shown in Figure~\ref{Fig_Com} demonstrate that our bound not only exhibits minimal dependence on model width, but also significantly outperforms existing bounds. Moreover, out of all the existing studies, ours is the only one that remains non-vacuous across all settings, consistently maintaining a value below $1$. This is remarkable since the networks are highly overparameterized. For example, for the MNIST dataset, the number of parameters is more than $md=2^{18}\cdot 1024\approx 2\cdot10^8$, which is much larger than the sample size $n=13007$. 

\section{Proof of Generalization Bounds\label{sec:proof-gen}}
\subsection{Necessary Lemmas}
To prove Theorem~\ref{thm:rad-snn-g}, we first introduce some necessary lemmas. The following lemma gives the  Efron-Stein inequality to control the variance of a general function of random variables.
\begin{lemma}[Efron-Stein Inequality\label{lem:efron}]
Consider $g:\zcal^n\mapsto\rbb$. Let $Z_1,\ldots,Z_n$ be independent random variables and $\widetilde{Z}=g(Z_1,\ldots,Z_n)$. Then
\[
\var(\widetilde{Z})\leq \sum_{i=1}^{n}\ebb\big[(\widetilde{Z}-\ebb_i[\widetilde{Z}])^2\big],
\]
where $\ebb_i[\widetilde{Z}]=\ebb\big[\widetilde{Z}|Z_1,\ldots,Z_{i-1},Z_{i+1},\ldots,Z_n\big]$.
\end{lemma}

The following lemma gives a contraction principle to remove nonlinear Lipschitz functions in estimating Rademacher complexities.
\begin{lemma}[Contraction Lemma~\citep{bartlett2002rademacher}\label{lem:rade-cont}]
  Suppose $\psi:\rbb\mapsto\rbb$ is $G$-Lipschitz. Then the following inequality holds for any $\widetilde{\fcal}$
  \[
  \ebb_{\bm{\sigma}}\sup_{f\in\widetilde{\fcal}}\sum_{i=1}^{n}\sigma_{i}\psi\big(f(x_i)\big)\leq G\ebb_{\bm{\sigma}}\sup_{f\in\widetilde{\fcal}}\sum_{i=1}^{n}\sigma_{i}f(x_i).
  \]
\end{lemma}
The following lemma uses the Efron-Stein Inequality to estimate an $\ell_2$ version of Rademacher complexities.
\begin{lemma}\label{lem:rad-square-a}
 Let $\gamma(\cdot)$ be $G_\gamma$-Lipschitz. For any $a\geq0$ and $\bw^{(0)}\in\rbb^d$, we have
  \[
  \ebb_{\bm{\sigma}}\sup_{\bw:\|\bw-\bw^{(0)}\|_2\leq a}\Big(\sum_{i=1}^{n}\sigma_{i}\big(\gamma(\bx_i^\top\bw)-\gamma(\bx_i^\top\bw^{(0)})\big)\Big)^2\leq 5a^2G^2_\gamma\sum_{i=1}^n\|\bx_i\|_2^2.
  \]
\end{lemma}
\begin{proof}
We define $g:\{+1,-1\}^n\mapsto \rbb$ as
\[
g(\sigma_1,\ldots,\sigma_n)=\sup_{\bw:\|\bw-\bw^{(0)}\|_2\leq a}\Big|\sum_{i=1}^{n}\sigma_{i}\big(\gamma(\bx_i^\top\bw)-\gamma(\bx_i^\top\bw^{(0)})\big)\Big|
\]
and $\widetilde{Z}=g(\sigma_1,\ldots,\sigma_n)$. For simplicity, we simplify $\sup_{\bw:\|\bw-\bw^{(0)}\|_2\leq a}[\cdot]$ as $\sup_{\bw}[\cdot]$ in the following analysis. We know
\begin{align*}
  & \widetilde{Z}-\ebb_k[\widetilde{Z}] = \sup_{\bw}\Big|\sum_{i=1}^{n}\sigma_{i}\big(\gamma(\bx_i^\top\bw)-\gamma(\bx_i^\top\bw^{(0)})\big)\Big| - \ebb_k\Big[\sup_{\bw}\Big|\sum_{i=1}^{n}\sigma_{i}\big(\gamma(\bx_i^\top\bw)-\gamma(\bx_i^\top\bw^{(0)})\big)\Big|\Big]\\
  & = \frac{1}{2}\Big(\sup_{\bw}\Big|\sum_{i=1}^{n}\sigma_{i}\big(\gamma(\bx_i^\top\bw)-\gamma(\bx_i^\top\bw^{(0)})\big)\Big| - \sup_{\bw}\Big|\sum_{i\in[n]:i\neq k}\sigma_{i}\big(\gamma(\bx_i^\top\bw)-\gamma(\bx_i^\top\bw^{(0)})\big)+\big(\gamma(\bx_k^\top\bw)-\gamma(\bx_k^\top\bw^{(0)})\big)\Big|\Big)\\
  & + \frac{1}{2}\Big(\sup_{\bw}\Big|\sum_{i=1}^{n}\sigma_{i}\big(\gamma(\bx_i^\top\bw)-\gamma(\bx_i^\top\bw^{(0)})\big)\Big| - \sup_{\bw}\Big|\sum_{i\in[n]:i\neq k}\sigma_{i}\big(\gamma(\bx_i^\top\bw)-\gamma(\bx_i^\top\bw^{(0)})\big)-\big(\gamma(\bx_k^\top\bw)-\gamma(\bx_k^\top\bw^{(0)})\big)\Big|\Big),
\end{align*}
where $\ebb_k[\cdot]=\ebb[\cdot|\sigma_j:j\neq k]$.
If $\sigma_k=1$, then we know
\begin{align*}
  & \big|\widetilde{Z}-\ebb_{k}[\widetilde{Z}]\big| \\
  & = \frac{1}{2}\Big|\sup_{\bw}\Big|\sum_{i=1}^{n}\sigma_{i}\big(\gamma(\bx_i^\top\bw)-\gamma(\bx_i^\top\bw^{(0)})\big)\Big| - \sup_{\bw}\Big|\sum_{i\in[n]:i\neq k}\sigma_{i}\big(\gamma(\bx_i^\top\bw)-\gamma(\bx_i^\top\bw^{(0)})\big)-\big(\gamma(\bx_k^\top\bw)-\gamma(\bx_k^\top\bw^{(0)})\big)\Big|\Big|\\
  & \leq \frac{1}{2}\sup_{\bw}\bigg|\Big|\sum_{i=1}^{n}\sigma_{i}\big(\gamma(\bx_i^\top\bw)-\gamma(\bx_i^\top\bw^{(0)})\big)\Big|-\Big|\sum_{i\in[n]:i\neq k}\sigma_{i}\big(\gamma(\bx_i^\top\bw)-\gamma(\bx_i^\top\bw^{(0)})\big)-\big(\gamma(\bx_k^\top\bw)-\gamma(\bx_k^\top\bw^{(0)})\big)\Big|\bigg|\\
  & \leq \frac{1}{2}\sup_{\bw}\Big|\gamma(\bx_k^\top\bw)-\gamma(\bx_k^\top\bw^{(0)})+\gamma(\bx_k^\top\bw)-\gamma(\bx_k^\top\bw^{(0)})\Big| \leq {G_\gamma}\sup_{\bw}\big|\bx_k^\top\bw-\bx_k^\top\bw^{(0)}\big|\leq aG_\gamma\|\bx_k\|_2,
\end{align*}
where we have used the Lipschitzness of $\gamma$ and the elementary inequality for any $h,\tilde{h}:\wcal\mapsto\rbb$
\begin{equation}\label{supsup}
\big|\sup_{\bw}h(\bw)-\sup_{\bw}\tilde{h}(\bw)\big| \leq \sup_{\bw}|h(\bw)-\tilde{h}(\bw)|.
\end{equation}
Similarly, we can also show that $\big|\widetilde{Z}-\ebb_{k}[\widetilde{Z}]\big|\leq aG_\gamma\|\bx_k\|_2$ if $\sigma_k=-1$. Therefore, we always have $\big|\widetilde{Z}-\ebb_{k}[\widetilde{Z}]\big|\leq aG_\gamma\|\bx_k\|_2$. We plug this inequality back into Lemma~\ref{lem:efron}, and derive
\[
\ebb\big[\widetilde{Z}^2\big]-\big(\ebb[\widetilde{Z}]\big)^2=\var(\widetilde{Z}) \leq \sum_{k=1}^{n}\ebb_{\bm{\sigma}}\big[(\widetilde{Z}-\ebb_k[\widetilde{Z}])^2\big]\leq a^2G^2_\gamma\sum_{i=1}^n\|\bx_i\|_2^2.
\]
Furthermore, by the standard result $\ebb\sup_{f\in\fcal} |\sum^n_{i=1}\sigma_i f(\bx_i)| \leq 2\ebb\sup_{f\in\{0\}\cup\fcal} \sum^n_{i=1}\sigma_i f(\bx_i)$~\citep{golowich2018size}
and Lemma~\ref{lem:rade-cont}, we know
\begin{align*}
  \ebb[\widetilde{Z}] &= \ebb\sup_{\bw}\Big|\sum_{i=1}^{n}\sigma_{i}\big(\gamma(\bx_i^\top\bw)-\gamma(\bx_i^\top\bw^{(0)})\big)\Big|
  \leq 2\ebb\sup_{\bw}\sum_{i=1}^{n}\sigma_{i}\big(\gamma(\bx_i^\top\bw)-\gamma(\bx_i^\top\bw^{(0)})\big)\\
  & = 2\ebb\sup_{\bw}\sum_{i=1}^{n}\sigma_{i}\gamma(\bx_i^\top\bw) \leq 2G_\gamma\ebb\sup_{\bw}\sum_{i=1}^{n}\sigma_{i}\bx_i^\top\bw \\
  & = 2G_\gamma\ebb\sup_{\bw}\sum_{i=1}^{n}\sigma_{i}\bx_i^\top(\bw -\bw^{(0)}) \leq  2G_\gamma\ebb\sup_{\bw}\Big\|\sum_{i=1}^{n}\sigma_{i}\bx_i\Big\|_2\|\bw -\bw^{(0)}\|_2\\
  & \leq 2aG_\gamma\ebb\Big\|\sum_{i=1}^{n}\sigma_{i}\bx_i\Big\|_2\leq 2aG_\gamma\big(\sum_{i=1}^{n}\|\bx_i\|_2^2\big)^{\frac{1}{2}},
\end{align*}
where we have used the following inequality due to the concavity of $x\mapsto\sqrt{x}$
\begin{equation}\label{rad-a-1}
\ebb_{\bm{\sigma}}\big\|\sum_{i=1}^{n}\sigma_i\bx_i\big\|_2 \leq \Big(\ebb_{\bm{\sigma}}\big\|\sum_{i=1}^{n}\sigma_i\bx_i\big\|_2^2\Big)^{\frac{1}{2}}
=\Big(\ebb_{\bm{\sigma}}\sum_{i=1}^{n}\sum_{j=1}^{n}\langle\sigma_i\bx_i,\sigma_j\bx_j\rangle\Big)^{\frac{1}{2}}=\big(\sum_{i=1}^{n}\|\bx_i\|_2^2\big)^{\frac{1}{2}}.
\end{equation}
We combine the above inequalities together and derive
\[
\ebb\big[\widetilde{Z}^2\big]\leq a^2G^2_\gamma\sum_{i=1}^n\|\bx_i\|_2^2+\Big(2aG_\gamma\big(\sum_{i=1}^{n}\|\bx_i\|_2^2\big)^{\frac{1}{2}}\Big)^2=5a^2G^2_\gamma\sum_{i=1}^n\|\bx_i\|_2^2.
\]
The proof is completed.
\end{proof}

The following lemma gives an estimate of the path-norm.
\begin{lemma}\label{lem:wv-sup}
  Let $\wcal$ and $\vcal$ be defined in Eq.~\eqref{WV}. Then, we have
  \[
  \sup_{\bW\in\wcal,\bV\in\vcal}\kappa(\bW,\bV)=c^{\frac{1}{2}}R_WR_V.
  \]
\end{lemma}
\begin{proof}
By the Schwarz's inequality, we derive
\begin{align}
  & \sum_{j=1}^{m}\sum_{k=1}^c|v_{kj}|\|\bw_j-\bw_j^{(0)}\|_2 \leq \Big(\sum_{j=1}^{m}\big(\sum_{k=1}^c|v_{kj}|\big)^2\Big)^{\frac{1}{2}}\Big(\sum_{j=1}^{m}\|\bw_j-\bw_j^{(0)}\|_2^2\Big)^{\frac{1}{2}}\notag\\
  & \leq \Big(c\sum_{j=1}^{m}\sum_{k=1}^cv^2_{kj}\Big)^{\frac{1}{2}}\|\bW-\bW^{(0)}\|_F =c^{\frac{1}{2}}\|\bW-\bW^{(0)}\|_F\|\bV\|_F\leq c^{\frac{1}{2}}R_WR_V.\label{wv-sup-1}
\end{align}
Furthermore, we can choose $\bV$ with $v_{kj}=R_V/\sqrt{cm}$ for all $k\in[c],j\in[m]$, and $\bW$ with $\|\bw_j-\bw_j^{(0)}\|_2=R_W/\sqrt{m}$ for all $j\in[m]$. It is clear that $\bW\in\wcal$ and $\bV\in\vcal$. Furthermore, we have
\[
\sum_{j=1}^{m}\sum_{k=1}^c|v_{kj}|\|\bw_j-\bw_j^{(0)}\|_2 = R_V\sum_{j=1}^{m}\frac{\sqrt{c}}{\sqrt{m}}\frac{R_W}{\sqrt{m}}=c^{\frac{1}{2}}R_WR_V.
\]
We combine the above two inequalities and get the stated bound.
\end{proof}

The following elementary lemma controls the expectation of a random variable in terms of its tail probability. We present the proof for completeness.
\begin{lemma}\label{lem:X-exp}
  Let $X$ be a random variable. Then, $\ebb[X]\leq \int_0^\infty \mathrm{Pr}\{X> a\}\mathrm{d}a$.
\end{lemma}
\begin{proof}
  Define $\widetilde{X}=\max\{X,0\}$. Then, by the standard result on the expectation of a nonnegative random variable, we know
  \[
  \ebb[X]\leq \ebb[\widetilde{X}]=\int_0^\infty \mathrm{Pr}\{\widetilde{X}> a\}\mathrm{d}a.
  \]
  It is clear that $\mathrm{Pr}\{\max\{X,0\}> a\}=\mathrm{Pr}\{X> a\}$ for any $a\geq0$. The stated bound then follows directly.
\end{proof}

\subsection{Multiple Rademacher Complexity}
Our analysis will encounter a variant of Rademacher complexity which we call the multiple Rademacher complexity.
The difference is that there is a maximum over $k\in[c]$.
If $c=1$, it is clear that it recovers the standard Rademacher complexity.
\begin{definition}[Multiple Rademacher Complexity]
Let $\gcal$ be a class of real-valued functions. Let $S=\{\bz_1,\ldots,\bz_n\}$, and $\sigma_{ik}$ be independent Rademacher variables for $i\in[n],k\in[c]$. Then, we define the multiple Rademacher complexity as
\[
\mathfrak{R}^c_{S}(\gcal)=\frac{1}{n}\ebb_{\bm{\sigma}}\sup_{g\in\gcal}\max_{k\in[c]}\sum_{i=1}^{n}\sigma_{ik}g(\bz_i).
\]
\end{definition}
The following lemma controls the multiple Rademacher complexity for a finite union of  function classes. The case with $c=1$ has been considered in \citep{maurer2014inequality}. The proof follows directly from \citep{maurer2014inequality}.
\begin{lemma}\label{lem:mauer}
Let $S=\{\bz_1,\ldots,\bz_n\}$ and $\fcal_j$ be a class of functions from $\zcal$ to $\rbb$ for each $j\in[m]$. Then
\[
\mathfrak{R}_S^c\big(\cup_{j\in[m]}\fcal_j\big)\leq \max_{j\in[m]}\mathfrak{R}_S(\fcal_j)+2\sqrt{2}\Big(\sup_{f\in\cup_j\fcal_j}\frac{1}{n}\sum_{i=1}^{n}f^2(\bz_i)\Big)^{\frac{1}{2}}\Big(\frac{\log (mc)}{n}\Big)^{\frac{1}{2}}
\Big(1+\frac{1}{2\log(mc)}\Big).
\]
\end{lemma}
\begin{proof}
For simplicity, we define $B=\big(\sup_{f\in\cup_j\fcal_j}\frac{1}{n}\sum_{i=1}^{n}f^2(\bz_i)\big)^{\frac{1}{2}}$. For any $j\in[m],k\in[c]$, define
\[
F_{jk}=\frac{1}{n}\sup_{f\in\fcal_j}\sum_{i=1}^{n}\sigma_{ik}f(\bz_i).
\]
Then, according to Theorem 4 in \citep{maurer2014inequality}, we have for any $s>0$
\[
\mathrm{Pr}\Big\{F_{jk}-\ebb_{\bm{\sigma}}[F_{j,k}]\geq s\Big\}\leq \exp\Big(-\frac{ns^2}{8B^2}\Big).
\]
A union bound shows that
\begin{align*}
\mathrm{Pr}\Big\{\max_{j\in[m],k\in[c]}F_{jk}-\max_{j\in[m],k\in[c]}\ebb_{\bm{\sigma}}[F_{j,k}]\geq s\Big\} \leq
mc\cdot\exp\Big(-\frac{ns^2}{8B^2}\Big).
\end{align*}
It then follows from Lemma~\ref{lem:X-exp} that
\begin{align*}
  & \ebb\big[\max_{j\in[m],k\in[c]}F_{jk}\big]  \leq \int_0^\infty \mathrm{Pr}\Big\{\max_{j\in[m],k\in[c]}F_{jk}>a\Big\}\mathrm{d}a \\
  & = \int_0^{\max_{j\in[m],k\in[c]}\ebb_{\bm{\sigma}}[F_{j,k}]+\delta}\mathrm{Pr}\Big\{\max_{j\in[m],k\in[c]}F_{jk}>a\Big\}\mathrm{d}a+\int^\infty_{\max_{j\in[m],k\in[c]}\ebb_{\bm{\sigma}}[F_{j,k}]+\delta}\mathrm{Pr}\Big\{\max_{j\in[m],k\in[c]}F_{jk}>a\Big\}\mathrm{d}a\\
  & \leq
  \max_{j\in[m],k\in[c]}\ebb_{\bm{\sigma}}[F_{j,k}]+\delta+\int^\infty_{\delta}\mathrm{Pr}\Big\{\max_{j\in[m],k\in[c]}F_{jk}>\max_{j\in[m],k\in[c]}\ebb_{\bm{\sigma}}[F_{j,k}]+a\Big\}\mathrm{d}a\\
  & \leq \max_{j\in[m],k\in[c]}\ebb_{\bm{\sigma}}[F_{j,k}]+\delta+mc\int^\infty_{\delta}\exp\Big(-\frac{na^2}{8B^2}\Big)\mathrm{d}a,
\end{align*}
where we have used the inequality that the probability of any event is no larger than $1$. It was shown that $\int_\delta^\infty\exp(-a^2/(2b))\mathrm{d}a\leq\frac{b}{\delta}\exp(-\delta^2/(2b))$ for any $b\geq0$ (Mill's inequality)~\citep{maurer2014inequality}. We apply this inequality and get
\[
\ebb\big[\max_{j\in[m],k\in[c]}F_{jk}\big]\leq  \max_{j\in[m],k\in[c]}\ebb_{\bm{\sigma}}[F_{j,k}]+\delta+\frac{mc\cdot 4B^2}{n\delta}\exp\Big(-\frac{\delta^2n}{8B^2}\Big).
\]
We choose $\delta=\frac{2B(2\log(mc))^{\frac{1}{2}}}{\sqrt{n}}$ and get
\begin{align*}
  \frac{mc\cdot 4B^2}{n\delta}\exp\Big(-\frac{\delta^2n}{8B^2}\Big) &= \frac{mc\cdot 4B^2}{\sqrt{2n}2B\log^{\frac{1}{2}}(mc)}\exp\Big(-\frac{8B^2\log(mc)n}{8B^2n}\Big)\\
  & = \frac{\sqrt{2}mcB}{\sqrt{n\log(mc)}}\exp\big(-\log(mc)\big) = \frac{\sqrt{2}B}{\sqrt{n\log(mc)}}.
\end{align*}
We combine the above two inequalities and get
\[
\ebb\big[\max_{j\in[m],k\in[c]}F_{jk}\big]\leq  \max_{j\in[m],k\in[c]}\ebb_{\bm{\sigma}}[F_{j,k}]+\frac{2\sqrt{2}B\log^{\frac{1}{2}}(mc)}{\sqrt{n}}+\frac{\sqrt{2}B}{\sqrt{n\log(mc)}}.
\]
The proof is completed.
\end{proof}
\subsection{Proof of on Rademacher Complexity Bounds\label{sec:proof-rad-snn-g}}
Before proving Theorem~\ref{thm:rad-snn-g}, we first introduce a lemma.
\begin{lemma}\label{lem:rad-snn-26}
Let $R_V\geq0$ and $\sigma_{ik}$ be independent Rademacher variables. We have
\begin{equation}\label{rad-snn-26-a}
\ebb_{\bm{\sigma}}\sup_{\bV:\|\bV\|_F\leq R_V}\sum_{i=1}^{n}\sum_{k=1}^c\sum_{j=1}^{m}v_{kj}\sigma_{ik}\gamma(\bx_i^\top\bw_j^{(0)})\leq
R_V\Big(c\sum_{j=1}^{m}\sum_{i=1}^{n}\gamma^2(\bx_i^\top\bw_j^{(0)})\Big)^{\frac{1}{2}}.
\end{equation}
\end{lemma}
\begin{proof}
By the identity $\sup_{\mathbf{a}:\|\mathbf{a}\|_2\leq R}\sum_{i=1}^{n}a_ib_i=(\sum_{i=1}^{n}b_i^2)^{\frac{1}{2}}R$, we know
\begin{align}
  & \ebb_{\bm{\sigma}}\sup_{\bV}\sum_{i=1}^{n}\sum_{k=1}^c\sum_{j=1}^{m}v_{kj}\sigma_{ik}\gamma(\bx_i^\top\bw_j^{(0)}) = \ebb_{\bm{\sigma}}\sup_{\bV}\sum_{k=1}^c\sum_{j=1}^{m}v_{kj}\sum_{i=1}^{n}\sigma_{ik}\gamma(\bx_i^\top\bw_j^{(0)})\notag\\
  & = \ebb_{\bm{\sigma}}\sup_{\bV}\big(\sum_{k=1}^c\sum_{j=1}^{m}v^2_{kj}\big)^{\frac{1}{2}}\Big(\sum_{k=1}^c\sum_{j=1}^{m}\Big(\sum_{i=1}^{n}\sigma_{ik}\gamma(\bx_i^\top\bw_j^{(0)})\Big)^2\Big)^{\frac{1}{2}}\label{rad-snn-26-1}\\
  & \leq R_V\Big(\sum_{k=1}^c\sum_{j=1}^{m}\ebb_{\bm{\sigma}}\Big(\sum_{i=1}^{n}\sigma_{ik}\gamma(\bx_i^\top\bw_j^{(0)})\Big)^2\Big)^{\frac{1}{2}},\notag
\end{align}
where we have used the concavity of $t\mapsto\sqrt{t}$. By the independency of $\sigma_{ik},i\in[n],k\in[c]$, we know $\ebb[\sigma_{ik}\sigma_{i'k}]=0$ if $i\neq i'$. It then follows the following identity for any $k\in[c],j\in[m]$
\[
\ebb_{\bm{\sigma}}\Big[\Big(\sum_{i=1}^{n}\sigma_{ik}\gamma(\bx_i^\top\bw_j^{(0)})\Big)^2\Big]
= \ebb_{\bm{\sigma}}\Big[\sum_{i=1}^{n}\sum_{i'=1}^{n}\sigma_{ik}\sigma_{i'k}\gamma(\bx_i^\top\bw_j^{(0)})\gamma(\bx_k^\top\bw_j^{(0)})\Big]
= \sum_{i=1}^{n}\gamma^2(\bx_i^\top\bw_j^{(0)}).
\]
The stated bound then follows directly by combining the above two inequalities together.
\end{proof}
\begin{proof}[Proof of Theorem~\ref{thm:rad-snn-g}]
Note for $g=\Psi_{\bW,\bV}$, we have $g_k(\bx)=\sum_{j\in[m]}v_{kj}\gamma(\bx^\top\bw_j)$. Then, by the definition of vector-valued Rademacher complexity, we know
\[
\mathfrak{R}_{S}(\gcal)=\frac{1}{n}\ebb_{\bm{\sigma}}\sup_{\bW,\bV}\sum_{i=1}^{n}\sum_{k=1}^c\sum_{j=1}^{m}\sigma_{ik}v_{kj}\gamma(\bx_i^\top\bw_j).
\]
Let $a\in(0,R_W)$ be a real number to be determined later.
We control the Rademacher complexity as follows
\begin{align}
  & n\mathfrak{R}_{S}(\gcal)-\ebb_{\bm{\sigma}}\sup_{\bV}\sum_{i=1}^{n}\sum_{k=1}^c\sum_{j=1}^{m}v_{kj}\sigma_{ik}\gamma(\bx_i^\top\bw_j^{(0)})
  \leq \ebb_{\bm{\sigma}}\sup_{\bW,\bV}\sum_{i=1}^{n}\sum_{k=1}^c\sum_{j=1}^{m}v_{kj}\sigma_{ik}\big(\gamma(\bx_i^\top\bw_j)-\gamma(\bx_i^\top\bw_j^{(0)})\big)\notag\\
  & = \ebb_{\bm{\sigma}}\sup_{\bW,\bV}\sum_{j=1}^{m}\sum_{k=1}^cv_{kj}\sum_{i=1}^{n}\big(\ibb_{[\|\bw_j-\bw_j^{(0)}\|_2\leq a]}+\ibb_{[\|\bw_j-\bw_j^{(0)}\|_2> a]}\big)\sigma_{ik}\big(\gamma(\bx_i^\top\bw_j)-\gamma(\bx_i^\top\bw_j^{(0)})\big)\notag\\
  & \leq \ebb_{\bm{\sigma}}\sup_{\bW,\bV}\sum_{j=1}^{m}\sum_{k=1}^cv_{kj}\sum_{i=1}^{n}\ibb_{[\|\bw_j-\bw_j^{(0)}\|_2\leq a]}\sigma_{ik}\big(\gamma(\bx_i^\top\bw_j)-\gamma(\bx_i^\top\bw_j^{(0)})\big)\notag\\
  & +\ebb_{\bm{\sigma}}\sup_{\bW,\bV}\sum_{j=1}^{m}\sum_{k=1}^cv_{kj}\sum_{i=1}^{n}\ibb_{[\|\bw_j-\bw_j^{(0)}\|_2> a]}\sigma_{ik}\big(\gamma(\bx_i^\top\bw_j)-\gamma(\bx_i^\top\bw_j^{(0)})\big),\label{rad-snn-1}
  \end{align}
  where we have used the subadditivity of supremum.
For the first term, by Schwarz's inequality and the concavity of $x\mapsto\sqrt{x}$, we know
\begin{align}
  & \ebb_{\bm{\sigma}}\sup_{\bW,\bV}\sum_{j=1}^{m}\sum_{k=1}^cv_{kj}\sum_{i=1}^{n}\ibb_{[\|\bw_j-\bw_j^{(0)}\|_2\leq a]}\sigma_{ik}\big(\gamma(\bx_i^\top\bw_j)-\gamma(\bx_i^\top\bw_j^{(0)})\big)\notag\\
  & \leq \ebb_{\bm{\sigma}}\sup_{\bW,\bV}\Big(\sum_{j=1}^{m}\sum_{k=1}^cv_{kj}^2\Big)^{\frac{1}{2}}\Big(\sum_{j=1}^{m}\sum_{k=1}^c\Big(\sum_{i=1}^{n}\ibb_{[\|\bw_j-\bw_j^{(0)}\|_2\leq a]}\sigma_{ik}\big(\gamma(\bx_i^\top\bw_j)-\gamma(\bx_i^\top\bw_j^{(0)})\big)\Big)^2\Big)^{\frac{1}{2}}\notag\\
  & \leq R_V\ebb_{\bm{\sigma}}\sup_{\bW}\Big(\sum_{j=1}^{m}\ibb_{[\|\bw_j-\bw_j^{(0)}\|_2\leq a]}\sum_{k=1}^c\Big(\sum_{i=1}^{n}\sigma_{ik}\big(\gamma(\bx_i^\top\bw_j)-\gamma(\bx_i^\top\bw_j^{(0)})\big)\Big)^2\Big)^{\frac{1}{2}}\notag\\
  & \leq R_V\Big(\sum_{j=1}^{m}\ebb_{\bm{\sigma}}\sup_{\bW}\ibb_{[\|\bw_j-\bw_j^{(0)}\|_2\leq a]}\sum_{k=1}^c\Big(\sum_{i=1}^{n}\sigma_{ik}\big(\gamma(\bx_i^\top\bw_j)-\gamma(\bx_i^\top\bw_j^{(0)})\big)\Big)^2\Big)^{\frac{1}{2}}\notag\\
  & \leq R_V\Big(\sum_{j=1}^{m}\sum_{k=1}^c\ebb_{\bm{\sigma}}\sup_{\bw:\|\bw-\bw_j^{(0)}\|_2\leq a}\Big(\sum_{i=1}^{n}\sigma_{ik}\big(\gamma(\bx_i^\top\bw)-\gamma(\bx_i^\top\bw_j^{(0)})\big)\Big)^2\Big)^{\frac{1}{2}}.\notag
\end{align}
It then follows from Lemma~\ref{lem:rad-square-a}  that
\begin{align}
  & \ebb_{\bm{\sigma}}\sup_{\bW,\bV}\sum_{j=1}^{m}\sum_{k=1}^cv_{kj}\sum_{i=1}^{n}\ibb_{[\|\bw_j-\bw_j^{(0)}\|_2\leq a]}\sigma_{ik}\big(\gamma(\bx_i^\top\bw_j)-\gamma(\bx_i^\top\bw_j^{(0)})\big)\notag\\
  & \leq R_V\Big(mc\cdot 5a^2G^2_\gamma\sum_{i=1}^n\|\bx_i\|_2^2\Big)^{\frac{1}{2}} = (5cm)^{\frac{1}{2}}aR_VG_\gamma\big(\sum_{i=1}^n\|\bx_i\|_2^2\big)^{\frac{1}{2}}.\label{rad-snn-2}
\end{align}
We now consider the second term in the decomposition~\eqref{rad-snn-1}, which can be bounded as follows
\begin{align}
  & \ebb_{\bm{\sigma}}\sup_{\bW,\bV}\sum_{j=1}^{m}\sum_{k=1}^cv_{kj}\sum_{i=1}^{n}\ibb_{[\|\bw_j-\bw_j^{(0)}\|_2> a]}\sigma_{ik}\big(\gamma(\bx_i^\top\bw_j)-\gamma(\bx_i^\top\bw_j^{(0)})\big) \notag\\
  & = \ebb_{\bm{\sigma}}\sup_{\bW,\bV}\sum_{j=1}^{m}\sum_{k=1}^cv_{kj}\|\bw_j-\bw_j^{(0)}\|_2\ibb_{[\|\bw_j-\bw_j^{(0)}\|_2> a]}\sum_{i=1}^{n}\sigma_{ik}\frac{\gamma(\bx_i^\top\bw_j)-\gamma(\bx_i^\top\bw_j^{(0)})}{\|\bw_j-\bw_j^{(0)}\|_2} \notag\\
  & \leq \ebb_{\bm{\sigma}}\sup_{\bW,\bV}\sum_{j=1}^{m}\sum_{k=1}^c|v_{kj}|\|\bw_j-\bw_j^{(0)}\|_2\max_{j\in[m]}\max_{k\in[c]}\ibb_{[\|\bw_j-\bw_j^{(0)}\|_2> a]}\Big|\sum_{i=1}^{n}\sigma_{ik}\frac{\gamma(\bx_i^\top\bw_j)-\gamma(\bx_i^\top\bw_j^{(0)})}{\|\bw_j-\bw_j^{(0)}\|_2}\Big|\notag\\
  & \leq \sup_{\bW,\bV}\kappa(\bW,\bV)\ebb_{\bm{\sigma}}\sup_{\bW}\max_{j\in[m]}\max_{k\in[c]}\ibb_{[\|\bw_j-\bw_j^{(0)}\|_2> a]}\Big|\sum_{i=1}^{n}\sigma_{ik}\frac{\gamma(\bx_i^\top\bw_j)-\gamma(\bx_i^\top\bw_j^{(0)})}{\|\bw_j-\bw_j^{(0)}\|_2}\Big|\notag\\
  & \leq  \Big(\sup_{\bW,\bV}\kappa(\bW,\bV)\Big)\ebb_{\bm{\sigma}}\max_{j\in[m]}\max_{k'\in[c]}\sup_{\bw\in\rbb^{d}:a<\|\bw-\bw_j^{(0)}\|_2\leq R_W}\Big|\sum_{i=1}^{n}\sigma_{ik'}\frac{\gamma(\bx_i^\top\bw_j)-\gamma(\bx_i^\top\bw_j^{(0)})}{\|\bw_j-\bw_j^{(0)}\|_2}\Big|,\label{rad-snn-33}
\end{align}
where we have used the inequality $\|\bw_j-\bw_j^{(0)}\|_2\leq \|\bW-\bW^{(0)}\|_F\leq R_W$.
Define $r_k=2^{k-1}a$ for $k=1,\ldots,K$, where $K=1+\lceil\log_2(R_W/a)\rceil$. Then, it is clear that $r_K\geq R_W$. For any $k\in[K]$ and $j\in[m]$, introduce
\begin{gather*}
\hcal_{k,j}=\Big\{\bx\mapsto \big(\gamma(\bx^\top\bw)-\gamma(\bx^\top\bw_j^{(0)})\big)/\|\bw-\bw_j^{(0)}\|_2:\bw\in\rbb^{d},
r_k<\|\bw-\bw_j^{(0)}\|_2\leq r_{k+1}\Big\},\\
\widetilde{\hcal}_{k,j}=\Big\{\bx\mapsto -\big(\gamma(\bx^\top\bw)-\gamma(\bx^\top\bw_j^{(0)})\big)/\|\bw-\bw_j^{(0)}\|_2:\bw\in\rbb^{d},
r_k<\|\bw-\bw_j^{(0)}\|_2\leq r_{k+1}\Big\}.
\end{gather*}
It then follows from Eq.~\eqref{rad-snn-33} and the definition of multiple Rademacher complexity that
\begin{multline}\label{rad-snn-3}
  \frac{1}{n}\ebb_{\bm{\sigma}}\sup_{\bW,\bV}\sum_{j=1}^{m}\sum_{k=1}^cv_{kj}\sum_{i=1}^{n}\ibb_{[\|\bw_j-\bw_j^{(0)}\|_2> a]}\sigma_{ik}\big(\gamma(\bx_i^\top\bw_j)-\gamma(\bx_i^\top\bw_j^{(0)})\big)\\
  \leq \sup_{\bW,\bV}\kappa(\bW,\bV)\cdot\mathfrak{R}_S^c\Big(\cup_{k\in[K],j\in[m]}\big(\hcal_{k,j}\cup\widetilde{\hcal}_{k,j}\big)\Big).
\end{multline}
For any $h\in\cup_{k\in[K],j\in[m]}(\hcal_{k,j}\cup\widetilde{\hcal}_{k,j})$ indexed by $\bw$, by the $G_\gamma$-Lipschitzness of $\gamma$, we know
\begin{align}
\frac{1}{n}\sum_{i=1}^{n}h^2(\bx_i)&
=\frac{1}{n}\sum_{i=1}^{n}\Big(\gamma(\bx_i^\top\bw)-\gamma(\bx_i^\top\bw_j^{(0)})\Big)^2/\|\bw-\bw_j^{(0)}\|_2^2
 \leq \frac{G_\gamma^2}{n}\sum_{i=1}^{n}\big(\bx_i^\top(\bw-\bw_j^{(0)})\big)^2/\|\bw-\bw_j^{(0)}\|_2^2\notag\\
& = \frac{G_\gamma^2}{n}(\bw-\bw_j^{(0)})^\top\sum_{i=1}^{n}\bx_i\bx_i^\top(\bw-\bw_j^{(0)}) /\|\bw-\bw_j^{(0)}\|_2^2
\leq G_\gamma^2\Big\|\frac{1}{n}\sum_{i=1}^{n}\bx_i\bx_i^\top\Big\|_\sigma,\label{rad-snn-6}
\end{align}
where the last inequality follows from the definition of the spectral norm.
Furthermore, there holds (note $x/b\leq\max\{x/b_1,0\}$ if $b\in[b_1,b_2]$ and $b_1>0$)
\begin{align*}
   n\mathfrak{R}_S\big(\hcal_{k,j}\big)  &= \ebb_{\bm{\sigma}}\sup_{\bw:r_k<\|\bw-\bw_j^{(0)}\|_2\leq r_{k+1} }\sum_{i=1}^{n}\sigma_i\big(\gamma(\bx_i^\top\bw)-\gamma(\bx_i^\top\bw_j^{(0)})\big)/\|\bw-\bw_j^{(0)}\|_2\\
  & \leq \ebb_{\bm{\sigma}}\sup_{\bw:r_k<\|\bw-\bw_j^{(0)}\|_2\leq r_{k+1} }\max\Big\{\sum_{i=1}^{n}\sigma_i\big(\gamma(\bx_i^\top\bw)-\gamma(\bx_i^\top\bw_j^{(0)})\big)/r_{k},0\Big\}\\
  & \leq \ebb_{\bm{\sigma}}\sup_{\bw:\|\bw-\bw_j^{(0)}\|_2\leq r_{k+1} }\max\Big\{\sum_{i=1}^{n}\sigma_i\big(\gamma(\bx_i^\top\bw)-\gamma(\bx_i^\top\bw_j^{(0)})\big)/r_{k},0\Big\}\\
  & = \ebb_{\bm{\sigma}}\sup_{\bw:\|\bw-\bw_j^{(0)}\|_2\leq r_{k+1} }\sum_{i=1}^{n}\sigma_i\big(\gamma(\bx_i^\top\bw)-\gamma(\bx_i^\top\bw_j^{(0)})\big)/r_{k},
\end{align*}
where the second identity holds since for any $\bm{\sigma}$ we have $\sup_{\bw:\|\bw-\bw_j^{(0)}\|_2\leq r_{k+1} }\sum_{i=1}^{n}\sigma_i\big(\gamma(\bx_i^\top\bw)-\gamma(\bx_i^\top\bw_j^{(0)})\big)\geq0$ as we can take $\bw=\bw_j^{(0)}$.
It then follows that
\begin{align*}
  n\mathfrak{R}_S\big(\hcal_{k,j}\big) & \leq \frac{1}{r_{k}}\ebb_{\bm{\sigma}}\sup_{\bw:\|\bw-\bw_j^{(0)}\|_2\leq r_{k+1} }\sum_{i=1}^{n}\sigma_i\big(\gamma(\bx_i^\top\bw)-\gamma(\bx_i^\top\bw_j^{(0)})\big)
  = \frac{1}{r_{k}}\ebb_{\bm{\sigma}}\sup_{\bw:\|\bw-\bw_j^{(0)}\|_2\leq r_{k+1} }\sum_{i=1}^{n}\sigma_i\gamma(\bx_i^\top\bw)\\
  & \leq  \frac{G_\gamma}{r_{k}}\ebb_{\bm{\sigma}}\sup_{\bw:\|\bw-\bw_j^{(0)}\|_2\leq r_{k+1} }\sum_{i=1}^{n}\sigma_i\bx_i^\top\bw
  =\frac{G_\gamma}{r_{k}}\ebb_{\bm{\sigma}}\sup_{\bw:\|\bw-\bw_j^{(0)}\|_2\leq r_{k+1} }\sum_{i=1}^{n}\sigma_i\bx_i^\top(\bw-\bw_j^{(0)})\\
  &= \frac{G_\gamma}{r_{k}}\ebb_{\bm{\sigma}}\sup_{\bw:\|\bw\|_2\leq r_{k+1} }\sum_{i=1}^{n}\sigma_i\bx_i^\top\bw
  \leq \frac{G_\gamma r_{k+1}}{r_{k}}\ebb_{\bm{\sigma}}\sup_{\bw:\|\bw\|_2\leq r_{k+1} }\big\|\sum_{i=1}^{n}\sigma_i\bx_i\big\|_2
  \leq 2G_\gamma \big(\sum_{i=1}^{n}\|\bx_i\|_2^2\big)^{\frac{1}{2}},
\end{align*}
where we have used Lemma~\ref{lem:rade-cont} and Eq.~\eqref{rad-a-1}.
Similarly, we also have
\[
n\mathfrak{R}_S\big(\widetilde{\hcal}_{k,j}\big)\leq 2G_\gamma \big(\sum_{i=1}^{n}\|\bx_i\|_2^2\big)^{\frac{1}{2}},\quad\forall j\in[m],k\in[K].
\]
We then plug Eq.~\eqref{rad-snn-6} and the above inequality back into Lemma~\ref{lem:mauer} (a finite union of $2mK$ function spaces), and get
\begin{align*}
  & \mathfrak{R}_S^c\Big(\cup_{k\in[K],j\in[m]}\big(\hcal_{k,j}\cup\widetilde{\hcal}_{k,j}\big)\Big)\\
  & \leq \frac{2G_\gamma}{n}\big(\sum_{i=1}^{n}\|\bx_i\|_2^2\big)^{\frac{1}{2}}+2\sqrt{2}\Big(G_\gamma^2\Big\|\frac{1}{n}\sum_{i=1}^{n}\bx_i\bx_i^\top\Big\|_\sigma\Big)^{\frac{1}{2}}\Big(\frac{\log (2mcK)}{n}\Big)^{\frac{1}{2}}\Big(1+\frac{1}{2\log(2mcK)}\Big).
\end{align*}
By Eq.~\eqref{rad-snn-3}, the second term in the decomposition~\eqref{rad-snn-1} can be bounded by
\begin{multline*}
\frac{1}{n}\ebb_{\bm{\sigma}}\sup_{\bW,\bV}\sum_{j=1}^{m}\sum_{k=1}^cv_{kj}\sum_{i=1}^{n}\ibb_{[\|\bw_j-\bw_j^{(0)}\|_2> a]}\sigma_{ik}\big(\gamma(\bx_i^\top\bw_j)-\gamma(\bx_i^\top\bw_j^{(0)})\big)\leq \\ G_\gamma\sup_{\bW,\bV}\kappa(\bW,\bV)\Big(\frac{2}{n}\big(\sum_{i=1}^{n}\|\bx_i\|_2^2\big)^{\frac{1}{2}}+2\sqrt{2}\Big\|\frac{1}{n}\sum_{i=1}^{n}\bx_i\bx_i^\top\Big\|_\sigma^{\frac{1}{2}}\Big(\frac{\log (2mcK)}{n}\Big)^{\frac{1}{2}}\Big(1+\frac{1}{2\log(2mcK)}\Big)\Big).
\end{multline*}
We plug the above inequality, Eq.~\eqref{rad-snn-26-a} and Eq.~\eqref{rad-snn-2} back into Eq.~\eqref{rad-snn-1}, and get
\begin{multline*}
\mathfrak{R}_{S}(\gcal)\leq \frac{R_V}{n}\Big(c\sum_{j=1}^{m}\sum_{i=1}^{n}\gamma^2(\bx_i^\top\bw_j^{(0)})\Big)^{\frac{1}{2}}+\frac{(5cm)^{\frac{1}{2}}aR_VG_\gamma}{n}\big(\sum_{i=1}^n\|\bx_i\|_2^2\big)^{\frac{1}{2}}+\\
G_\gamma\sup_{\bW,\bV}\kappa(\bW,\bV)\Big(\frac{2}{n}\big(\sum_{i=1}^{n}\|\bx_i\|_2^2\big)^{\frac{1}{2}}+2\sqrt{2}\Big\|\frac{1}{n}\sum_{i=1}^{n}\bx_i\bx_i^\top\Big\|_\sigma^{\frac{1}{2}}\Big(\frac{\log (2mcK)}{n}\Big)^{\frac{1}{2}}\Big(1+\frac{1}{2\log(2mcK)}\Big)\Big).
\end{multline*}
We choose $a=\sup_{\bW,\bV}\kappa(\bW,\bV)/(cmR_V^2)^{\frac{1}{2}}$ and derive
\begin{multline*}
\mathfrak{R}_{S}(\gcal)\leq \frac{R_V}{n}\Big(c\sum_{j=1}^{m}\sum_{i=1}^{n}\gamma^2(\bx_i^\top\bw_j^{(0)})\Big)^{\frac{1}{2}}+\\ G_\gamma\sup_{\bW,\bV}\kappa(\bW,\bV)\Big(\frac{2+\sqrt{5}}{n}\big(\sum_{i=1}^{n}\|\bx_i\|_2^2\big)^{\frac{1}{2}}+2\sqrt{2}\Big\|\frac{1}{n}\sum_{i=1}^{n}\bx_i\bx_i^\top\Big\|_\sigma^{\frac{1}{2}}\Big(\frac{\log (2mcK)}{n}\Big)^{\frac{1}{2}}\Big(1+\frac{1}{2\log(2mcK)}\Big)\Big).
\end{multline*}
Eq.~\eqref{rad-snn-g-a} follows since $K=1+\lceil\log_2(R_W/a)\rceil$ and $a=\sup_{\bW,\bV}\kappa(\bW,\bV)/(cmR_V^2)^{\frac{1}{2}}$.

We now turn to Eq.~\eqref{rad-snn-g-b}. According to Lemma~\ref{lem:wv-sup}, we know $\sup_{\bW,\bV}\kappa(\bW,\bV)=c^{\frac{1}{2}}R_WR_V$. We plug this choice back into Eq.~\eqref{rad-snn-g-a}, and get Eq.~\eqref{rad-snn-g-b}.
\end{proof}

We now prove Theorem~\ref{thm:lower} on lower bounds of Rademacher complexities.
\begin{proof}[Proof of Theorem~\ref{thm:lower}]
 Define $\mathbb{B}=\{\bw\in\rbb^d:\|\bw\|_2\leq R_W-r_0\}$ and
\[
\gcal_2=\Big\{\bx\mapsto R_V\gamma(\bw^\top\bx):\bw\in \mathbb{B}\Big\}.
\]
We first show that $\gcal_2\subseteq\gcal$. Indeed, let $j^*=\arg\min_{j\in[m]}\|\bw^{(0)}_j\|_2$. Then, for any $\bw'$ with $\|\bw'\|_2\leq R_W-r_0$, we can define $\bW=(\bw_1^\top,\ldots,\bw_m^\top)^\top$ and $\bV=(v_1,\ldots,v_m)$ as follows
\[
\bw_j=\begin{cases}
        \bw_j^{(0)}, & \mbox{if } j\neq j^* \\
        \bw', & \mbox{otherwise}
      \end{cases}\quad\text{and}\quad v_j=\begin{cases}
                             0, & \mbox{if } j\neq j^* \\
                             R_V, & \mbox{otherwise}.
                           \end{cases}
\]
It then follows that
\[
\Psi_{\bW,\bV}(\bx)=\sum_{j\in[m]}v_j\gamma(\bx^\top\bw_j)=v_{j^*}\gamma(\bx^\top\bw_{j^*})=R_V\gamma(\bx^\top\bw').
\]
Furthermore, it is clear that
\[
\|\bW-\bW^{(0)}\|_F=\Big(\sum_{j\in[m]}\|\bw_j-\bw_j^{(0)}\|_2^2\Big)^{\frac{1}{2}}=\|\bw_{j^*}-\bw_{j^*}^{(0)}\|_2\leq \|\bw'\|_2+\|\bw_{j^*}^{(0)}\|_2\leq R_W-r_0+r_0= R_W.
\]
This shows that any function in $\gcal_2$ can be realized by a function in $\gcal$. It then follows that $\mathfrak{R}_S(\gcal_2)\leq\mathfrak{R}_S(\gcal)$. We now control $\mathfrak{R}_S(\gcal_2)$ from below. For any $a\in\rbb$, we define $a_+=\max\{a,0\}$ and $a_-=\max\{-a,0\}$. It is clear that $a=a_+-a_-$ for any $a$. Therefore, we know
\begin{align*}
  \ebb_{\sigma}\sup_{\bw\in \mathbb{B}}\sum_{i=1}^{n}\sigma_i\bw^\top\bx_i & = \ebb_{\sigma}\sup_{\bw\in \mathbb{B}}\sum_{i=1}^{n}\sigma_i\Big((\bw^\top\bx_i)_+-(\bw^\top\bx_i)_-\Big)= \ebb_{\sigma}\sup_{\bw\in \mathbb{B}}\Big(\sum_{i=1}^{n}\sigma_i(\bw^\top\bx_i)_+-\sum_{i=1}^{n}\sigma_i(\bw^\top\bx_i)_-\Big)\\
  & \leq \ebb_{\sigma}\sup_{\bw\in \mathbb{B}}\sum_{i=1}^{n}\sigma_i(\bw^\top\bx_i)_++\ebb_{\sigma}\sup_{\bw}-\sum_{i=1}^{n}\sigma_i(\bw^\top\bx_i)_- \\
  & = \ebb_{\sigma}\sup_{\bw\in \mathbb{B}}\sum_{i=1}^{n}\sigma_i\max\{\bw^\top\bx_i,0\}+\ebb_{\sigma}\sup_{\bw\in \mathbb{B}}\sum_{i=1}^{n}\sigma_i\max\{-\bw^\top\bx_i,0\},
\end{align*}
where we have used the subadditivity of supremum and the symmetry of Rademacher variables. Note that $\bw\in \mathbb{B}$ implies that $-\bw\in \mathbb{B}$. Therefore, we have
\[
\ebb_{\sigma}\sup_{\bw\in \mathbb{B}}\sum_{i=1}^{n}\sigma_i\max\{\bw^\top\bx_i,0\}=\ebb_{\sigma}\sup_{\bw\in \mathbb{B}}\sum_{i=1}^{n}\sigma_i\max\{-\bw^\top\bx_i,0\}.
\]
It then follows that
\begin{align*}
  2\ebb_{\sigma}\sup_{\bw\in \mathbb{B}}\sum_{i=1}^{n}\sigma_i\max\{\bw^\top\bx_i,0\} & \geq \ebb_{\sigma}\sup_{\bw\in \mathbb{B}}\sum_{i=1}^{n}\sigma_i\bw^\top\bx_i
  = \ebb_{\sigma}\sup_{\bw\in \mathbb{B}}\bw^\top\Big(\sum_{i=1}^{n}\sigma_i\bx_i\Big) \\
  &  = \ebb_{\sigma}\sup_{\bw\in \mathbb{B}}\|\bw\|_2\Big\|\sum_{i=1}^{n}\sigma_i\bx_i\Big\|_2= (R_W-r_0)\ebb_{\sigma}\Big\|\sum_{i=1}^{n}\sigma_i\bx_i\Big\|_2 \\
  & \geq \frac{R_W-r_0}{\sqrt{2}}\big(\sum_{i=1}^{n}\|\bx_i\|_2^2\big)^{\frac{1}{2}},
\end{align*}
where we have used the Khitchine-Kahane inequality $\ebb_\sigma\|\sum_{i=1}^{n}\sigma_{i}\bv_{i}\|_2\geq 2^{-\frac{1}{2}}\big(\sum_{i=1}^{n}\|\bv_{i}\|_2^2\big)^{\frac{1}{2}}$.
The third equality above holds since $\mathbb{B}$ is an Euclidean ball centered at the origin and therefore the supremum is attained when $\bw$ and $\bu$ are aligned, yielding
$
\sup_{\bw\in\mathbb{B}}\bw^\top\bu
= \sup_{\bw\in\mathbb{B}}\|\bw\|_2\|\bu\|_2$.
 This shows
\begin{align}\label{lower-1}
  \mathfrak{R}_S(\gcal)\geq\mathfrak{R}_S(\gcal_2) = \frac{R_V}{n}\ebb_{\sigma}\sup_{\bw\in \mathbb{B}}\sum_{i=1}^{n}\sigma_i\max\{\bw^\top\bx_i,0\} \geq \frac{(R_W-r_0)R_V}{2\sqrt{2}n}\big(\sum_{i=1}^{n}\|\bx_i\|_2^2\big)^{\frac{1}{2}}.
\end{align}
Define
\[
\gcal_3=\Big\{\bx\mapsto \Psi_{\bW^{(0)},\bV}(\bx):\bV\in \vcal\Big\}.
\]
It is clear that $\mathfrak{R}_S(\gcal)\geq\mathfrak{R}_S(\gcal_3)$. Denote $M_i=\big(\gamma(\bx_i^\top\bw_1^{(0)}), \dotsc, \gamma(\bx_i^\top\bw_m^{(0)})\big)^\top\in\rbb^m$ and therefore $\|\sum^n_{i=1}\sigma_iM_i\|_2^2=\sum^m_{j=1}\big(\sum^n_{i=1}\sigma_i\gamma(\bx_i^\top\bw_j^{(0)})\big)^2$. Using these notations, we get
\begin{align*}
  \mathfrak{R}_S(\gcal_3) &= \frac{1}{n}\ebb_{\bm{\sigma}}\sup_{\bV}\sum_{i=1}^{n}\sum_{j=1}^{m}v_{j}\sigma_{i}\gamma(\bx_i^\top\bw_j^{(0)})
  = \frac{1}{n}\ebb_{\bm{\sigma}}\sup_{\bV}\big(\sum_{j=1}^{m}v^2_{j}\big)^{\frac{1}{2}}\Big(\sum_{j=1}^{m}\Big(\sum_{i=1}^{n}\sigma_{i}\gamma(\bx_i^\top\bw_j^{(0)})\Big)^2\Big)^{\frac{1}{2}} \\ &=\frac{R_V}{n}\ebb_{\bm{\sigma}}\|\sum^n_{i=1}\sigma_iM_i\|_2,
\end{align*}
where we have used Eq.~\eqref{rad-snn-26-1} in the second identity. According to Khitchine-Kahane inequality, we further get
\begin{equation*}
\mathfrak{R}_S(\gcal_3)\geq \frac{R_V}{\sqrt{2}n}\big(\sum_{i=1}^{n}\|M_{i}\|_2^2\big)^{\frac{1}{2}} = \frac{R_V}{\sqrt{2}n}\big(\sum_{i=1}^{n}\sum^m_{j=1}\gamma^2(\bx_i^\top\bw_j^{(0)})\big)^{\frac{1}{2}}.
\end{equation*}
We now have two lower bounds on $\mathfrak{R}_S(\gcal)$.
The proof is completed by combining these two lower bounds together.
\end{proof}


\subsection{Proof of Theorem~\ref{thm:generalization}\label{sec:proof-thm-generalization}}
In this section, we prove Theorem~\ref{thm:generalization} on generalization bounds. To this aim, we first introduce a contraction lemma on vector-valued Rademacher complexities. The factor $\sqrt{2}$ can be removed if $c=1$.
\begin{lemma}[\citep{maurer2016vector}\label{lem:maurer-contraction}]
  Let $S=\{\bz_i\}_{i=1}^n\in\zcal^n$. Let $\fcal$ be a class of functions $f:\zcal\mapsto\rbb^c$ and $h_i:\rbb^c\mapsto\rbb$ be $G$-Lipschitz. Then
  \[
  \ebb_{\bm{\sigma}\sim\{\pm1\}^n}\Big[\sup_{f\in\fcal}\sum_{i\in[n]}\sigma_i(h_i\circ f)(\bz_i)\Big]
  \leq\sqrt{2}G\ebb_{\bm{\sigma}\sim\{\pm1\}^{nc}}\Big[\sup_{f\in\fcal}\sum_{i\in[n]}\sum_{j\in[c]}\sigma_{ij}f_j(\bz_i)\Big].
  \]
\end{lemma}
We then apply Lemma~\ref{lem:maurer-contraction} to derive generalization bounds when learning with a fixed hypothesis space
\begin{equation}\label{gcal-new}
\gcal_{R_W,R_V,R_\kappa}:=\Big\{\bx\mapsto\Psi_{\bW,\bV}(\bx):(\bW,\bV)\in\kcal_{R_W,R_V,R_\kappa}\Big\},
\end{equation}
where
\[
\kcal_{R_W,R_V,R_\kappa}=\Big\{(\bW;\bV)\in\rbb^{m\times d}\times\rbb^{c\times m}:\|\bW-\bW^{(0)}\|_F\leq R_W,\|\bV\|_F\leq R_V,\kappa(\bW,\bV)\leq R_\kappa\Big\}.
\]
\begin{proposition}\label{prop:gen-snn-fix}
  Assume $\ell_y(\cdot)$ is $G$-Lipschitz continuous and $\ell_y(\Psi_{\bW,\bV}(\bx))\in[0,b]$ almost surely. Let $\gcal$ be defined in Eq.~\eqref{gcal}. Then, with probability at least $1-\delta$ the following inequality holds uniformly for all $(\bW,\bV)\in\gcal_{R_W,R_V,R_\kappa}$
  \[
  F(\bW,\bV)-F_S(\bW,\bV)\leq 2\sqrt{2}GG_\gamma R_\kappa\Big(\frac{2+\sqrt{5}}{n}\big(\sum_{i=1}^{n}\|\bx_i\|_2^2\big)^{\frac{1}{2}}+\frac{c_m'\big\|\sum_{i=1}^{n}\bx_i\bx_i^\top\big\|_\sigma^{\frac{1}{2}}}{n}\Big)+3b\Big(\frac{\log(2/\delta)}{2n}\Big)^{\frac{1}{2}},
  \]
  where $c_m'=2\sqrt{2}\big(1+\frac{1}{2\log_2(2mc)}\big)\log^{\frac{1}{2}}\big(2mc\big\lceil\log_2\max\big\{\frac{2R_WR_V(cm)^{\frac{1}{2}}}{R_\kappa},2m^{\frac{1}{2}}\big\}\big\rceil\big)$.
\end{proposition}
\begin{proof}
  For brevity, we let $\kcal:=\kcal_{R_W,R_V,R_\kappa}$.
  A standard result in Rademacher complexity analysis gives the following bound with probability at least $1-\delta$~\citep{neyshabur2019towards}
  \begin{multline}
  \sup_{(\bW,\bV)\in\kcal}\big[F(\bW,\bV)-F_S(\bW,\bV)\big]\leq
  2\mathfrak{R}_S\Big(\Big\{(\bx,y)\mapsto \ell_y(\Psi_{\bW,\bV}(\bx)):(\bW,\bV)\in\kcal\Big\}\Big)+3b\Big(\frac{\log(2/\delta)}{2n}\Big)^{\frac{1}{2}}.\label{gen-snn-fix-0}
  \end{multline}
  By the $G$-Lipschitzness of $\ell$, we can control $\mathfrak{R}_S\big(\big\{(\bx,y)\mapsto \ell_y(\Psi_{\bW,\bV}(\bx)):(\bW,\bV)\in\kcal\big\}\big)$ by Lemma~\ref{lem:maurer-contraction} as follows
  \begin{equation}\label{gen-snn-fix-1}
  \mathfrak{R}_S\Big((\bx,y)\mapsto \ell_y(\Psi_{\bW,\bV}(\bx)):(\bW,\bV)\in\kcal\Big)\leq \sqrt{2}G\mathfrak{R}_S\big(\gcal_{R_W,R_V,R_\kappa}\big).
  \end{equation}
  By Eq.~\eqref{rad-snn-g-a}, we know that
  \begin{equation}\label{gen-snn-fix-2}
  \mathfrak{R}_S\big(\gcal_{R_W,R_V,R_\kappa}\big)\leq \frac{R_V}{n}\Big(c\sum_{j=1}^{m}\sum_{i=1}^{n}\gamma^2(\bx_i^\top\bw_j^{(0)})\Big)^{\frac{1}{2}}+ G_\gamma\sup_{\bW,\bV}\kappa(\bW,\bV)\Big(\frac{2+\sqrt{5}}{n}\big(\sum_{i=1}^{n}\|\bx_i\|_2^2\big)^{\frac{1}{2}}+\frac{c_m\big\|\sum_{i=1}^{n}\bx_i\bx_i^\top\big\|_\sigma^{\frac{1}{2}}}{n}\Big),
  \end{equation}
  where $c_m=2\sqrt{2}\big(1+\frac{1}{2\log(2mc)}\big)\log^{\frac{1}{2}}\big(2mc\big\lceil\log_2\big(2R_WR_V(cm)^{\frac{1}{2}}/\sup_{\bW,\bV}\kappa(\bW,\bV)\big)\big\rceil\big)$. We now control $\sup_{\bW,\bV}\kappa(\bW,\bV)$ from below by considering two cases.
  \begin{itemize}
    \item If $R_\kappa\geq c^{\frac{1}{2}}R_WR_V$, then Lemma~\ref{lem:wv-sup} shows that
  $\sup_{\bW\in\wcal,\bV\in\vcal}\kappa(\bW,\bV)=c^{\frac{1}{2}}R_WR_V$ (the constraint $\kappa(\bW,\bV)\leq R_\kappa$ does not take effect in this case by Eq.~\eqref{wv-sup-1}).
    \item We now consider the case $R_\kappa< c^{\frac{1}{2}}R_WR_V$. In this case, we can choose
  $\bW$ with $\|\bw_j-\bw_j^{(0)}\|_2=R_W/\sqrt{m}$ and $\bV$ with $v_{kj}=R_\kappa/(\sqrt{m}cR_W)$ for all $j\in[m],k\in[c]$. Then we know
  \[
  \kappa(\bW,\bV)= \sum_{j=1}^{m}\sum_{k=1}^c\frac{R_\kappa}{\sqrt{m}cR_W}\frac{R_W}{\sqrt{m}}=R_\kappa.
  \]
  Furthermore, it is clear that $\|\bW-\bW^{(0)}\|_F=R_W$ and
  \[
  \|\bV\|_F=\sqrt{mc}\frac{R_\kappa}{\sqrt{m}cR_W}=\frac{R_\kappa}{\sqrt{c}R_W}\leq R_V.
  \]
  That is, the above constructed $(\bW,\bV)\in\kcal$ and therefore $\sup_{(\bW,\bV)}\kappa(\bW,\bV)= R_\kappa$.
  \end{itemize}
  We combine the above two cases, and get that $\sup_{(\bW,\bV)}\kappa(\bW,\bV)=\min\big\{R_\kappa,c^{\frac{1}{2}}R_WR_V\big\}$ and
  \begin{align}
    \frac{R_W(cmR_V^2)^{\frac{1}{2}}}{\sup_{\bW,\bV}\kappa(\bW,\bV)}& = \max\Big\{\frac{R_W(cmR_V^2)^{\frac{1}{2}}}{R_\kappa},\frac{R_W(cmR_V^2)^{\frac{1}{2}}}{c^{\frac{1}{2}}R_WR_V}\Big\}
     = \max\Big\{\frac{R_W(cmR_V^2)^{\frac{1}{2}}}{R_\kappa},m^{\frac{1}{2}}\Big\}.\notag
  \end{align}
  We plug the above inequality back into Eq.~\eqref{gen-snn-fix-2} and get
  \[
  \mathfrak{R}_S\big(\gcal_{R_W,R_V,R_\kappa}\big)\leq \frac{R_V}{n}\Big(c\sum_{j=1}^{m}\sum_{i=1}^{n}\gamma^2(\bx_i^\top\bw_j^{(0)})\Big)^{\frac{1}{2}}+G_\gamma R_\kappa\Big(\frac{2+\sqrt{5}}{n}\big(\sum_{i=1}^{n}\|\bx_i\|_2^2\big)^{\frac{1}{2}}+\frac{c_m'\big\|\sum_{i=1}^{n}\bx_i\bx_i^\top\big\|_\sigma^{\frac{1}{2}}}{n}\Big).
  \]
  We combine the above inequality, Eq.~\eqref{gen-snn-fix-0} and Eq.~\eqref{gen-snn-fix-1} together, and get
  \begin{align*}
  \sup_{\bW,\bV}\big[F(\bW,\bV)-F_S(\bW,\bV)\big]  & \leq 2\sqrt{2}G\mathfrak{R}_S\big(\gcal_{R_W,R_V,R_\kappa}\big)+3b\Big(\frac{\log(2/\delta)}{2n}\Big)^{\frac{1}{2}}\\
  & \leq \frac{2\sqrt{2}GR_V}{n}\Big(c\sum_{j=1}^{m}\sum_{i=1}^{n}\gamma^2(\bx_i^\top\bw_j^{(0)})\Big)^{\frac{1}{2}}+3b\Big(\frac{\log(2/\delta)}{2n}\Big)^{\frac{1}{2}}\\
  & +
  2\sqrt{2}GG_\gamma R_\kappa\Big(\frac{2+\sqrt{5}}{n}\big(\sum_{i=1}^{n}\|\bx_i\|_2^2\big)^{\frac{1}{2}}+\frac{c_m'\big\|\sum_{i=1}^{n}\bx_i\bx_i^\top\big\|_\sigma^{\frac{1}{2}}}{n}\Big).
  \end{align*}
  The proof is completed.
\end{proof}

\begin{proof}[Proof of Theorem~\ref{thm:generalization}]
  For any $r_1,r_2,r_3\in\nbb$, define
  \begin{equation}\label{gcal-r}
  \gcal_{r_1,r_2,r_3}=\Big\{\bx\mapsto\Psi_{\bW,\bV}(\bx):\|\bW-\bW^{(0)}\|_F\leq r_1,\|\bV\|_F\leq r_2,\kappa(\bW,\bV)\leq r_3\Big\}
  \end{equation}
  and
  \[
  c_{r_1,r_2}=2\sqrt{2}\Big(1+\frac{1}{2\log(2mc)}\Big)\log^{\frac{1}{2}}\big(2mc\big\lceil\log_2\max\big\{{2r_1r_2(cm)^{\frac{1}{2}}},2m^{\frac{1}{2}}\big\}\big\rceil\big).
  \]
  For any $r_1,r_2,r_3\in\nbb$, we apply Proposition~\ref{prop:gen-snn-fix} with $\gcal=\gcal_{r_1,r_2,r_3}$ to get the following inequality with probability at least $1-\delta/(r_1(r_1+1)r_2(r_2+1)r_3(r_3+1))$
  \begin{multline}\label{union-1}
  \sup_{\bW,\bV:\|\bW-\bW^{(0)}\|_F\leq r_1,\|\bV\|_F\leq r_2,\kappa(\bW,\bV)\leq r_3}\;\big[F(\bW,\bV)-F_S(\bW,\bV)\big]\leq \frac{2\sqrt{2}Gr_2}{n}\Big(c\sum_{j=1}^{m}\sum_{i=1}^{n}\gamma^2(\bx_i^\top\bw_j^{(0)})\Big)^{\frac{1}{2}}\\ +2\sqrt{2}GG_\gamma r_3\Big(\frac{2+\sqrt{5}}{n}\big(\sum_{i=1}^{n}\|\bx_i\|_2^2\big)^{\frac{1}{2}}+\frac{c_{r_1,r_2}\big\|\sum_{i=1}^{n}\bx_i\bx_i^\top\big\|_\sigma^{\frac{1}{2}}}{n}\Big)+3b\Big(\frac{\log(2r_1(r_1+1)r_2(r_2+1)r_3(r_3+1)/\delta)}{2n}\Big)^{\frac{1}{2}},
  \end{multline}
  where we use the inequality $r_3\geq1$ to control $c_m'$ in Proposition~\ref{prop:gen-snn-fix}.
  By the union bounds of probability, we know that Eq.~\eqref{union-1} holds simultaneously for all $r_1,r_2,r_3\in\nbb$ with probability at least
  \begin{align}
  & 1-\sum_{r_1\in\nbb}\sum_{r_2\in\nbb}\sum_{r_3\in\nbb}\frac{\delta}{r_1(r_1+1)r_2(r_2+1)r_3(r_3+1)} \notag\\
  & = 1 - \delta \Big(\sum_{r_1\in\nbb}\frac{1}{r_1(r_1+1)}\Big)\Big(\sum_{r_2\in\nbb}\frac{1}{r_2(r_2+1)}\Big)\Big(\sum_{r_3\in\nbb}\frac{1}{r_3(r_3+1)}\Big)\notag\\
  & \geq 1-\delta.\label{union-bound-1}
  \end{align}
  We now assume that Eq.~\eqref{union-1} holds simultaneously for all $r_1,r_2,r_3\in\nbb$. 
  For any $\bW,\bV$, let $r_1',r_2'$ and $r_3'$ be the smallest indices such that $\Psi_{\bW,\bV}\in\gcal_{r_1',r_2',r_3'}$. Then, it is clear that
  \begin{equation}\label{union-2}
  r_1'-1\leq \|\bW-\bW^{(0)}\|_F\leq r_1',\quad r_2'-1\leq \|\bV\|_F\leq r_2',\quad r_3'-1\leq \kappa(\bW,\bV)\leq r_3'.
  \end{equation}
  It then follows that
  \begin{align}
  & F(\bW,\bV)-F_S(\bW,\bV) \notag\\
  & \leq \frac{2\sqrt{2}Gr_2'}{n}\Big(c\sum_{j=1}^{m}\sum_{i=1}^{n}\gamma^2(\bx_i^\top\bw_j^{(0)})\Big)^{\frac{1}{2}}+2\sqrt{2}GG_\gamma r_3'\Big(\frac{2\!+\!\sqrt{5}}{n}\big(\sum_{i=1}^{n}\|\bx_i\|_2^2\big)^{\frac{1}{2}}\!+\!\frac{c_{r_1',r_2'}\big\|\sum_{i=1}^{n}\bx_i\bx_i^\top\big\|_\sigma^{\frac{1}{2}}}{n}\Big)\notag\\
  & + 3b\Big(\frac{\log(2r'_1(r'_1\!+\!1)r'_2(r'_2\!+\!1)r'_3(r'_3\!+\!1)/\delta)}{2n}\Big)^{\frac{1}{2}}
  \leq \frac{2\sqrt{2}G(\|\bV\|_F+1)}{n}\Big(c\sum_{j=1}^{m}\sum_{i=1}^{n}\gamma^2(\bx_i^\top\bw_j^{(0)})\Big)^{\frac{1}{2}}\notag\\
  & + 2\sqrt{2}GG_\gamma (\kappa(\bW,\bV)+1)\Big(\frac{2+\sqrt{5}}{n}\big(\sum_{i=1}^{n}\|\bx_i\|_2^2\big)^{\frac{1}{2}}+\frac{c_{\|\bW-\bW^{(0)}\|_F+1,\|\bV\|_F+1}\big\|\sum_{i=1}^{n}\bx_i\bx_i^\top\big\|_\sigma^{\frac{1}{2}}}{n}\Big)+\notag\\
  & 3b\Big(\frac{\log(2(\|\bW\!-\!\bW^{(0)}\|_F\!+\!1)(\|\bW\!-\!\bW^{(0)}\|_F\!+\!2)(\|\bV\|_F\!+\!1)(\|\bV\|_F\!+\!2)(\kappa(\bW,\bV)\!+\!1)(\kappa(\bW,\bV)\!+\!2)/\delta)}{2n}\Big)^{\frac{1}{2}}.\label{constant-bound}
  \end{align}
  The stated bound then follows directly by noting that $\|\bX\|_F=\big(\sum_{i=1}^n\|\bx_i\|_2^2\big)^{\frac{1}{2}}$ and $\big\|\sum_{i=1}^{n}\bx_i\bx_i^\top\big\|_\sigma\leq \|\bX\|_F^2$.
\end{proof}

\section{Conclusion\label{sec:conclusion}}
In this paper, we present initialization-dependent generalization bounds for SNNs with a general Lipschitz activation function. Our bound is data-dependent and use path-norm to measure the distance from initialization, while the existing initialization-dependent analyses give bounds in terms of the product of the Frobenius norm. Our empirical analyses show that our generalization bounds are significantly sharper than existing bounds, especially if the width is large.
Unlike existing lower bounds focusing on $\bW^{(0)}=0$, we also develop lower bounds for initialization-dependent SNNs. Our upper and lower bounds match up to a logarithmic factor. We make a comprehensive comparison with existing generalization bounds for SNNs, which shows a consistent improvement.

There remain several interesting questions for further studies. First, we only consider SNNs in our generalization analysis. A natural direction is to investigate whether our analyses can be extended to develop initialization-dependent bounds for DNNs. Second, our lower bounds are established for ReLU networks. It is interesting to develop similar lower bounds for neural networks with general Lipschitz activation functions. Third, we only consider fully connected neural networks. It is interesting to investigate initialization-dependent bounds for other neural networks such as convolutional neural networks.

\section*{Acknowledgements}
We are grateful to Ohad Shamir for valuable suggestions and comments.

\setlength{\bibsep}{0.03cm}
\bibliographystyle{abbrvnat}
\small
\bibliography{learning}

@article{bartlett2002rademacher,
  author    = {Peter Bartlett and
               Shahar Mendelson},
  title     = {Rademacher and Gaussian Complexities: Risk Bounds and Structural
               Results},
  journal   = {Journal of Machine Learning Research},
  volume    = {3},
  year      = {2002},
  pages     = {463-482},
  ee        = {http://www.jmlr.org/papers/v3/bartlett02a.html},
  bibsource = {DBLP, http://dblp.uni-trier.de}
}

@inproceedings{maurer2014inequality,
  title={An inequality with applications to structured sparsity and multitask dictionary learning},
  author={Maurer, Andreas and Pontil, Massimiliano and Romera-Paredes, Bernardino},
  booktitle={Conference on Learning Theory},
  pages={440--460},
  year={2014},
  organization={PMLR}
}

@inproceedings{neyshabur2019towards,
  title={Towards understanding the role of over-parametrization in generalization of neural networks},
  author={Neyshabur, Behnam and Li, Zhiyuan and Bhojanapalli, Srinadh and LeCun, Yann and Srebro, Nathan},
  booktitle={International Conference on Learning Representations},
  year={2019}
}

@inproceedings{neyshabur2018pac,
  title={A PAC-Bayesian Approach to Spectrally-Normalized Margin Bounds for Neural Networks},
  author={Neyshabur, Behnam and Bhojanapalli, Srinadh and Srebro, Nathan},
  booktitle={International Conference on Learning Representations},
  year={2018}
}

@article{dziugaite2017computing,
  title={Computing nonvacuous generalization bounds for deep (stochastic) neural networks with many more parameters than training data},
  author={Dziugaite, Gintare Karolina and Roy, Daniel M},
  journal={arXiv preprint arXiv:1703.11008},
  year={2017}
}

@article{magen2023initialization,
  title={Initialization-dependent sample complexity of linear predictors and neural networks},
  author={Magen, Roey and Shamir, Ohad},
  journal={Advances in Neural Information Processing Systems},
  volume={36},
  pages={7632--7658},
  year={2023}
}

@inproceedings{neyshabur2015norm,
  title={Norm-based capacity control in neural networks},
  author={Neyshabur, Behnam and Tomioka, Ryota and Srebro, Nathan},
  booktitle={Conference on Learning Theory},
  pages={1376--1401},
  year={2015},
  organization={PMLR}
}

@inproceedings{chen2021much,
  title={How much over-parameterization is sufficient to learn deep ReLU networks?},
  author={Chen, Zixiang and Cao, Yuan and Zou, Difan and Gu, Quanquan},
  booktitle={International Conference on Learning Representation},
  year={2021}
}

@inproceedings{ji2019polylogarithmic,
  title={Polylogarithmic width suffices for gradient descent to achieve arbitrarily small test error with shallow ReLU networks},
  author={Ji, Ziwei and Telgarsky, Matus},
  booktitle={International Conference on Learning Representations},
  year={2019}
}

@article{zhang2021understanding,
  title={Understanding deep learning (still) requires rethinking generalization},
  author={Zhang, Chiyuan and Bengio, Samy and Hardt, Moritz and Recht, Benjamin and Vinyals, Oriol},
  journal={Communications of the ACM},
  volume={64},
  number={3},
  pages={107--115},
  year={2021}
}

@article{soltanolkotabi2018theoretical,
  title={Theoretical insights into the optimization landscape of over-parameterized shallow neural networks},
  author={Soltanolkotabi, Mahdi and Javanmard, Adel and Lee, Jason D},
  journal={IEEE Transactions on Information Theory},
  volume={65},
  number={2},
  pages={742--769},
  year={2018},
  publisher={IEEE}
}

@article{ji2021understanding,
  title={Understanding estimation and generalization error of generative adversarial networks},
  author={Ji, Kaiyi and Zhou, Yi and Liang, Yingbin},
  journal={IEEE Transactions on Information Theory},
  volume={67},
  number={5},
  pages={3114--3129},
  year={2021},
  publisher={IEEE}
}

@inproceedings{maurer2016vector,
  title={A vector-contraction inequality for Rademacher complexities},
  author={Maurer, Andreas},
  booktitle={International Conference on Algorithmic Learning Theory},
  pages={3--17},
  year={2016}
}

@article{lecun2015deep,
  title={Deep learning},
  author={LeCun, Yann and Bengio, Yoshua and Hinton, Geoffrey},
  journal={Nature},
  volume={521},
  number={7553},
  pages={436--444},
  year={2015},
  publisher={Nature Publishing Group}
}

@inproceedings{golowich2018size,
  title={Size-independent sample complexity of neural networks},
  author={Golowich, Noah and Rakhlin, Alexander and Shamir, Ohad},
  booktitle={Conference On Learning Theory},
  pages={297--299},
  year={2018},
  organization={PMLR}
}

@inproceedings{hardt2016train,
  title={Train faster, generalize better: Stability of stochastic gradient descent},
  author={Hardt, Moritz and Recht, Ben and Singer, Yoram},
  booktitle={International Conference on Machine Learning},
  pages={1225--1234},
  year={2016}
}

@inproceedings{bartlett2017spectrally,
  title={Spectrally-normalized margin bounds for neural networks},
  author={Bartlett, Peter L and Foster, Dylan J and Telgarsky, Matus J},
  booktitle={Advances in Neural Information Processing Systems},
  pages={6240--6249},
  year={2017}
}

@article{richards2021stability,
  title={Stability \& generalisation of gradient descent for shallow neural networks without the neural tangent kernel},
  author={Richards, Dominic and Kuzborskij, Ilja},
  journal={Advances in neural information processing systems},
  volume={34},
  pages={8609--8621},
  year={2021}
}

@InProceedings{lei2023generalization,
  title = 	 {Generalization Analysis for Contrastive Representation Learning},
  author =       {Lei, Yunwen and Yang, Tianbao and Ying, Yiming and Zhou, Ding-Xuan},
  booktitle = 	 {International Conference on Machine Learning},
  pages = 	 {19200--19227},
  year = 	 {2023}
}

@article{oymak2020toward,
  title={Toward moderate overparameterization: global convergence guarantees for training shallow neural networks},
  author={Oymak, Samet and Soltanolkotabi, Mahdi},
  journal={IEEE Journal on Selected Areas in Information Theory},
  volume={1},
  number={1},
  pages={84--105},
  year={2020}
}

@inproceedings{arora2019fine,
  title={Fine-Grained Analysis of Optimization and Generalization for Overparameterized Two-Layer Neural Networks},
  author={Arora, Sanjeev and Du, Simon and Hu, Wei and Li, Zhiyuan and Wang, Ruosong},
  booktitle={International Conference on Machine Learning},
  pages={322--332},
  year={2019}
}

@inproceedings{lei2020fine,
  title={Fine-Grained Analysis of Stability and Generalization for Stochastic Gradient Descent},
  author={Lei, Yunwen and Ying, Yiming},
  booktitle={International Conference on Machine Learning},
  pages={5809--5819},
  year={2020}
}

@book{vapnik2013nature,
  title={The nature of statistical learning theory},
  author={Vapnik, Vladimir},
  year={2013},
  publisher={Springer}
}

@article{bousquet2002stability,
  title={Stability and generalization},
  author={Bousquet, Olivier and Elisseeff, Andr{\'e}},
  journal={Journal of Machine Learning Research},
  volume={2},
  number={Mar},
  pages={499--526},
  year={2002}
}

@article{li2020complexityb,
  title={Complexity measures for neural networks with general activation functions using path-based norms},
  author={Li, Zhong and Ma, Chao and Wu, Lei},
  journal={arXiv preprint arXiv:2009.06132},
  year={2020}
}

@article{lei2026optimization,
  title={Optimization and generalization of gradient descent for shallow relu networks with minimal width},
  author={Lei, Yunwen and Wang, Puyu and Ying, Yiming and Zhou, Ding-Xuan},
  journal={Journal of Machine Learning Research},
  volume={27},
  number={34},
  pages={1--35},
  year={2026}
}

@article{ma2019implicit,
  title={Implicit Regularization in Nonconvex Statistical Estimation: Gradient Descent Converges Linearly for Phase Retrieval, Matrix Completion, and Blind Deconvolution},
  author={Ma, Cong and Wang, Kaizheng and Chi, Yuejie and Chen, Yuxin},
  journal={Foundations of Computational Mathematics},
  volume={20},
  pages={451–-632},
  year={2019}
}

@inproceedings{li2025optimal,
  title={Optimal Rates for Generalization of Gradient Descent for Deep ReLU Classification},
  author={Li, Yuanfan and Lei, Yunwen and Guo, Zheng-Chu and Ying, Yiming},
  booktitle={Advances in Neural Information Processing Systems},
  year={2025}
}

@inproceedings{du2018gradient,
  title={Gradient Descent Provably Optimizes Over-parameterized Neural Networks},
  author={Du, Simon S and Zhai, Xiyu and Poczos, Barnabas and Singh, Aarti},
  booktitle={International Conference on Learning Representations},
  year={2018}
}

@inproceedings{suzuki2020compression,
  title={Compression based bound for non-compressed network: unified generalization error analysis of large compressible deep neural network},
  author={Suzuki, Taiji and Abe, Hiroshi and Nishimura, Tomoaki},
  booktitle={International Conference on Learning Representations},
  year={2020}
}

@article{chuang2021measuring,
  title={Measuring generalization with optimal transport},
  author={Chuang, Ching-Yao and Mroueh, Youssef and Greenewald, Kristjan and Torralba, Antonio and Jegelka, Stefanie},
  journal={Advances in Neural Information Processing Systems},
  volume={34},
  pages={8294--8306},
  year={2021}
}

@article{bartlett2019nearly,
  title={Nearly-tight VC-dimension and pseudodimension bounds for piecewise linear neural networks},
  author={Bartlett, Peter L and Harvey, Nick and Liaw, Christopher and Mehrabian, Abbas},
  journal={Journal of Machine Learning Research},
  volume={20},
  number={63},
  pages={1--17},
  year={2019}
}

@inproceedings{allen2019learning,
  title={Learning and generalization in overparameterized neural networks, going beyond two layers},
  author={Allen-Zhu, Zeyuan and Li, Yuanzhi and Liang, Yingyu},
  booktitle={Advances in Neural Information Processing Systems},
  pages={6158--6169},
  volume={32},
  year={2019}
}

@article{liu2024learning,
  title={Learning with norm constrained, over-parameterized, two-layer neural networks},
  author={Liu, Fanghui and Dadi, Leello and Cevher, Volkan},
  journal={Journal of Machine Learning Research},
  volume={25},
  number={138},
  pages={1--42},
  year={2024}
}

@article{yang2024nonparametric,
  title={Nonparametric regression using over-parameterized shallow ReLU neural networks},
  author={Yang, Yunfei and Zhou, Ding-Xuan},
  journal={Journal of Machine Learning Research},
  volume={25},
  number={165},
  pages={1--35},
  year={2024}
}

@article{neyshabur2015path,
  title={Path-{SGD}: Path-normalized optimization in deep neural networks},
  author={Neyshabur, Behnam and Salakhutdinov, Russ R and Srebro, Nati},
  journal={Advances in Neural Information Processing Systems},
  pages={2422--2430},
  volume={28},
  year={2015}
}

@article{parhi2021banach,
  title={Banach space representer theorems for neural networks and ridge splines},
  author={Parhi, Rahul and Nowak, Robert D},
  journal={Journal of Machine Learning Research},
  volume={22},
  number={43},
  pages={1--40},
  year={2021}
}

@article{taheri2024generalization,
  author  = {Hossein Taheri and Christos Thrampoulidis},
  title   = {Generalization and Stability of Interpolating Neural Networks with Minimal Width},
  journal = {Journal of Machine Learning Research},
  year    = {2024},
  volume  = {25},
  number  = {156},
  pages   = {1--41}
}

@inproceedings{taheri2025sharper,
  title={Sharper Guarantees for Learning Neural Network Classifiers with Gradient Methods},
  author={Taheri, Hossein and Thrampoulidis, Christos and Mazumdar, Arya},
  booktitle={International Conference on Learning Representations},
  year = {2025}
}

@inproceedings{ledent2021norm,
  title={Norm-based generalisation bounds for deep multi-class convolutional neural networks},
  author={Ledent, Antoine and Mustafa, Waleed and Lei, Yunwen and Kloft, Marius},
  booktitle={Proceedings of the AAAI Conference on Artificial Intelligence},
  volume={35},
  pages={8279--8287},
  year={2021}
}

@article{liu2020linearity,
  title={On the linearity of large non-linear models: when and why the tangent kernel is constant},
  author={Liu, Chaoyue and Zhu, Libin and Belkin, Misha},
  journal={Advances in Neural Information Processing Systems},
  volume={33},
  pages={15954--15964},
  year={2020}
}

@article{wang2025generalization,
  title={Generalization guarantees of gradient descent for shallow neural networks},
  author={Wang, Puyu and Lei, Yunwen and Wang, Di and Ying, Yiming and Zhou, Ding-Xuan},
  journal={Neural Computation},
  volume={37},
  number={2},
  pages={344--402},
  year={2025}
}

@inproceedings{arora2018stronger,
  title={Stronger generalization bounds for deep nets via a compression approach},
  author={Arora, Sanjeev and Ge, Rong and Neyshabur, Behnam and Zhang, Yi},
  booktitle={International Conference on Machine Learning},
  pages={254--263},
  year={2018},
  organization={PMLR}
}

@article{he2025information,
  title={Information-theoretic generalization bounds for deep neural networks},
  author={He, Haiyun and Goldfeld, Ziv},
  journal={IEEE Transactions on Information Theory},
    volume={71},
  number={8},
  pages={6227--6247},
  year={2025}
}

@article{bartlett2020benign,
  title={Benign overfitting in linear regression},
  author={Bartlett, Peter L and Long, Philip M and Lugosi, G{\'a}bor and Tsigler, Alexander},
  journal={Proceedings of the National Academy of Sciences},
  volume={117},
  number={48},
  pages={30063--30070},
  year={2020}
}

@article{nagarajan2019generalization,
  title={Generalization in deep networks: The role of distance from initialization},
  author={Nagarajan, Vaishnavh and Kolter, J Zico},
  journal={arXiv preprint arXiv:1901.01672},
  year={2019}
}

@article{chen2020generalized,
  title={A generalized neural tangent kernel analysis for two-layer neural networks},
  author={Chen, Zixiang and Cao, Yuan and Gu, Quanquan and Zhang, Tong},
  journal={Advances in Neural Information Processing Systems},
  volume={33},
  pages={13363--13373},
  year={2020}
}

@article{ding2025semi,
  title={Semi-supervised deep sobolev regression: Estimation and variable selection by requ neural network},
  author={Ding, Zhao and Duan, Chenguang and Jiao, Yuling and Yang, Jerry Zhijian},
  journal={IEEE Transactions on Information Theory},
  year={2025},
  publisher={IEEE}
}

@article{weinan2022barron,
  title={The Barron Space and the Flow-Induced Function Spaces for Neural Network Models},
  author={Weinan, E and Ma, Chao and Wu, Lei},
  journal={Constructive Approximation},
  volume={55},
  number={1},
  pages={369--406},
  year={2022}
}

@inproceedings{li2023transformers,
  title={Transformers as algorithms: Generalization and stability in in-context learning},
  author={Li, Yingcong and Ildiz, Muhammed Emrullah and Papailiopoulos, Dimitris and Oymak, Samet},
  booktitle={International Conference on Machine Learning},
  pages={19565--19594},
  year={2023}
}

@article{deora2024optimization,
  title={On the Optimization and Generalization of Multi-head Attention},
  author={Deora, Puneesh and Ghaderi, Rouzbeh and Taheri, Hossein and Thrampoulidis, Christos},
  journal={Transactions on Machine Learning Research},
  year={2024}
}

@inproceedings{daniely2024sample,
  title={On the sample complexity of two-layer networks: Lipschitz vs. element-wise lipschitz activation},
  author={Daniely, Amit and Granot, Elad},
  booktitle={International Conference on Algorithmic Learning Theory},
  pages={505--517},
  year={2024},
  organization={PMLR}
}

\end{document}